\documentclass{article}
\PassOptionsToPackage{square, numbers, compress}{natbib}

\usepackage[final,main]{neurips_2026}

\usepackage[utf8]{inputenc} 
\usepackage[T1]{fontenc}    
\usepackage{hyperref}       
\usepackage{url}            
\usepackage{booktabs}       
\usepackage{amsfonts}       
\usepackage{nicefrac}       
\usepackage{microtype}      
\usepackage{xcolor}         

\usepackage{mathtools}
\usepackage[capitalize]{cleveref}
\usepackage{xspace}
\usepackage{algorithm}
\usepackage{algpseudocode}
\usepackage{xparse}
\usepackage{amsthm}
\usepackage{amsmath}
\usepackage[createShortEnv,conf={restate, no link to proof, text proof=Proof}]{proof-at-the-end}
\usepackage[textsize=tiny]{todonotes}
\usepackage{acro}
\usepackage{tabularx}
\usepackage[makeroom]{cancel}
\usepackage{siunitx}
\usepackage{csquotes}
\usepackage{makecell}
\usepackage{wrapfig}
\usepackage{subcaption}
\usepackage{placeins}
\usepackage{enumitem}
\usepackage{multirow}
\usepackage{multicol}

 \NewAcroTemplate[list]{tabularx}{%
\AcroNeedPackage{tabularx}%
\acronymsmapF{%
\AcroAddRow{
\textbf{\acrowrite{short}}%
\acroifT{alt}{/\acrowrite{alt}}
&
\acrowrite{list}%
\acroifanyT{foreign,extra}{ (}%
\acroifT{foreign}{\acrowrite{foreign}\acroifT{extra}{, }}%
\acroifT{extra}{\acrowrite{extra}}%
 \acroifanyT{foreign,extra}{)}%
\acropagefill
\acropages
{\acrotranslate{page}\nobreakspace}
{\acrotranslate{pages}\nobreakspace}%
 \tabularnewline
 }%
 }
 {\AcroRerun}%
 \acroheading
\acropreamble
\par\noindent
\begin{tabularx}{\linewidth}{lX}
\AcronymTable
\end{tabularx}
}

\DeclareAcronym{licq}{short={LICQ}, long={Linear Independent Constraint Qualifier}, extra={see \Cref{ass:licq}}}
\DeclareAcronym{mcmc}{short=MCMC, long={Markov Chain Monte Carlo}}
\DeclareAcronym{smc}{short=SMC, long={Sequential Monte Carlo}}
\DeclareAcronym{scmc}{short=SCMC, long={Sequentially Constrained Monte Carlo}, extra={constrained sampler, introduced in \citep{golchi2016sequentially}}}
\DeclareAcronym{nhr}{short=NHR, long={Non-linear hit \& run}, extra={constrained sampler, introduced in \citep{toussaint2024nlp}}}
\DeclareAcronym{olla}{short=OLLA, long={Overdamped Langevin with Landing}, extra={constrained sampler, introduced in \citep{jeon2025fast}}}
\DeclareAcronym{masem}{short=MASEM, long={Manifold Sampling via Entropy Maximization}, extra={introduced in this paper}}
\DeclareAcronym{svdg}{short=SVGD, long={Stein Variational Gradient Descent}, extra={introduced in \citep{qiang2016stein}}}
\DeclareAcronym{CGF}{short=CGF, long=Cumulant Generating Function}

\newcommand{\licq}{%
  \hyperref[ass:licq]{\ac{licq}}\xspace%
}

  \renewcommand{\a}{\alpha}
  \renewcommand{\b}{\beta}
  \renewcommand{\d}{\delta}

  \renewcommand{\l}{\lambda}
  
    \newcommand{\m}{\mu}

  \renewcommand{\k}{\kappa}
  \renewcommand{\o}{\omega}
  \renewcommand{\O}{\Omega}

  \renewcommand{\r}{\varrho}
    \newcommand{\s}{\sigma}
  
  \renewcommand{\t}{\theta}

  \renewcommand{\SS}{{\cal S}}

    \newcommand{\Hent}{{\text H}}

  \newcommand{\RRR}{{\mathbb{R}}}

  \newcommand{\vx}{\mathbf{x}}

  \renewcommand{\|}{\,|\,}
  
  \renewcommand{\=}{\!=\!}

  \newcommand{\tsum}{\textstyle\sum}

  \renewcommand{\vec}{\boldsymbol}
  \providecommand{\url}[2]{\href{#1}{{\color{blue}#2}}}

  \newcommand{\fullhide}[1]{\message{^^JHIDE--Warning (hidden)!^^J}}
  \newcommand{\mytitle}[1]{\title{#1}\newcommand{\thetitle}{#1}}
  \newcommand{\header}{\begin{document}\mytitle\cleardefs}
  \newcommand{\contents}{{\tableofcontents}\renewcommand{\contents}{}}
  
  \newcommand{\widepaper}{\usepackage{geometry}\geometry{a4paper,hdivide={25mm,*,25mm},vdivide={25mm,*,25mm}}}
  \newcommand{\moviex}[2]{\movie[externalviewer]{#1}{#2}} 
  \newcommand{\rbox}[1]{\fboxrule2mm\fcolorbox[rgb]{1,.85,.85}{1,.85,.85}{#1}}
  \newcommand{\mpage}[2]{{\begin{minipage}{#1\columnwidth}#2\end{minipage}}}
  \newcommand{\redbox}[2]{\fboxrule1mm\fcolorbox[rgb]{1,.7,.7}{1,.7,.7}{\begin{minipage}{#1\columnwidth}\center#2\end{minipage}}}
  \newcommand{\onecol}[2]{
    \begin{minipage}[c]{#1\columnwidth}#2\end{minipage}}
  \newcommand{\twocol}[5][0]{
    \begin{minipage}[c]{#2\columnwidth}#4\end{minipage}\hspace*{#1\columnwidth}%
    \begin{minipage}[c]{#3\columnwidth}#5\end{minipage}}
  \newcommand{\threecol}[7][0]{%
    \begin{minipage}[c]{#2\columnwidth}#5\end{minipage}\hspace*{#1\columnwidth}%
    \begin{minipage}[c]{#3\columnwidth}#6\end{minipage}\hspace*{#1\columnwidth}%
    \begin{minipage}[c]{#4\columnwidth}#7\end{minipage}}
  \newcommand{\threecoltext}[7][c]{
    \begin{minipage}[#1]{#2\textwidth}#5\end{minipage}%
    \begin{minipage}[#1]{#3\textwidth}#6\end{minipage}%
    \begin{minipage}[#1]{#4\textwidth}#7\end{minipage}}
  \newcommand{\threecoltop}[7][0]{%
   \begin{minipage}[t]{#2\columnwidth}#5\end{minipage}\hspace*{#1\columnwidth}%
   \begin{minipage}[t]{#3\columnwidth}#6\end{minipage}\hspace*{#1\columnwidth}%
   \begin{minipage}[t]{#4\columnwidth}#7\end{minipage}}
  \newcommand{\fourcol}[9][0]{%
   \begin{minipage}[c]{#2\columnwidth}#6\end{minipage}\hspace*{#1\columnwidth}%
   \begin{minipage}[c]{#3\columnwidth}#7\end{minipage}\hspace*{#1\columnwidth}%
   \begin{minipage}[c]{#4\columnwidth}#8\end{minipage}\hspace*{#1\columnwidth}%
   \begin{minipage}[c]{#5\columnwidth}#9\end{minipage}}
  \newcommand{\helvetica}[4]{\setlength{\unitlength}{1pt}\fontsize{#1}{#1}\linespread{#2}\usefont{OT1}{phv}{#3}{#4}}
  \newcommand{\helve}[1]{\helvetica{#1}{1.5}{m}{n}}
  \newcommand{\german}{\usepackage[german]{babel}\usepackage[utf8]{inputenc}}
\newcommand{\sans}{\renewcommand{\familydefault}{\sfdefault}}

\newcommand{\norm}[1]{|\!|#1|\!|}
\newcommand{\expr}[1]{[\hspace{-.2ex}[#1]\hspace{-.2ex}]}

\newcommand{\Jwi}{J^\sharp_W}
\newcommand{\THi}{T^\sharp_H}
\newcommand{\Jci}{J^\natural_C}
\newcommand{\hJi}{{\bar J}^\sharp}
\renewcommand{\|}{\,|\,}
\renewcommand{\=}{\!=\!}
\newcommand{\myminus}{{\hspace*{-.0pt}\text{\rm -}\hspace*{-.5pt}}}
\newcommand{\myplus}{{\hspace*{-.0pt}\text{\rm +}\hspace*{-.5pt}}}
\newcommand{\1}{{\myminus1}}
\newcommand{\2}{{\myminus2}}
\newcommand{\3}{{\myminus3}}
\newcommand{\mT}{{\text{\rm -}\hspace*{-1pt}\top}}
\newcommand{\po}{{\myplus1}}
\newcommand{\pt}{{\myplus2}}
\newcommand{\T}{{\!\top\!}}
\newcommand{\xT}{{\underline x}}
\newcommand{\uT}{{\underline u}}
\newcommand{\zT}{{\underline z}}
\newcommand{\Sum}{\textstyle\sum}
\newcommand{\Int}{\textstyle\int}
\newcommand{\Prod}{\textstyle\prod}

\newenvironment{centy}{
\vspace{15mm}
\large
\hspace*{5mm}
\begin{minipage}{8cm}\it\color{blue}
}{
\end{minipage}
}

\newcommand{\old}{{\text{old}}}
\newcommand{\MAP}{{\text{MAP}}}
\newcommand{\ML}{{\text{ML}}}

\newcommand{\pub}[1]{{\color{green}\helvetica{8}{1.3}{m}{n}#1\\}}
\newcommand{\KL}[2]{\opKL\big(#1\,\big|\!\big|\,#2\big)} 

\renewcommand{\show}[2][.7]{\centerline{\includegraphics[width=#1\columnwidth]{#2}}}
\newcommand{\showh}[2][.8]{\includegraphics[width=#1\columnwidth]{#2}}
\newcommand{\shows}[2][.8]{\centerline{\includegraphics[scale=#1]{#2}}}
\newcommand{\showhs}[2][.8]{\includegraphics[scale=#1]{#2}}
\newcommand{\mov}[2]{\movie[externalviewer]{{\color{blue}\small #1}}{movies/#2}}
\newcommand{\movex}[2]{\movie[externalviewer]{#1}{#2}} 
\newcommand{\movh}[3][autostart,loop]{
\movie[#1]{\showh[#2]{#3.jpg}}{#3.mp4}%
\movie[externalviewer]{$\circ$}{#3.mp4}
}
\newcommand{\movc}[3][autostart,loop]{\centerline{\movh[#1]{#2}{#3}}}
\newcommand{\mymovie}[3][]{\centerline{\movie[autostart,loop,#1]{\includegraphics[#1]{#2}}{#3}}}

\newcommand{\cen}[1]{\centerline{#1}}

\newcommand{\citing}[1]{
{\color{citcol}\ttiny#1\par}
}

\newcommand{\cit}[3]{
\par\smallskip
{\color{citcol}\ttiny #1: \emph{#2}. #3 \par}
}

\newcommand{\citurl}[4]{
\par\smallskip
{\color{citcol}\ttiny #1: \protect{\href{#4}{\color{blue}{#2.}}} #3 \par}
}

\newcommand{\cito}[3]{
\par\smallskip
{\color{bluecol}\ttiny #1: \emph{#2}. #3 \par}
}

\newcommand{\redoMacrosInProof}{
  \renewcommand{\d}{\delta}
  \renewcommand{\=}{\!=\!}
}

\newcommand{\pdflatex}{
  \definecolor{bluecol}{rgb}{0,0,.5}
  \DeclareGraphicsExtensions{.pdf,.png,.jpg,.eps}
  \renewcommand{\r}{\varrho}
  \renewcommand{\l}{\lambda}
  \renewcommand{\s}{\sigma}
  \renewcommand{\b}{\beta}
  \renewcommand{\d}{\delta}
  \renewcommand{\k}{\kappa}
  \renewcommand{\t}{\theta}
  \renewcommand{\O}{\Omega}
  \renewcommand{\o}{\omega}
  \renewcommand{\SS}{{\cal S}}
  \renewcommand{\=}{\!=\!}
}

\definecolor{plotblue}{HTML}{1f7bb9}

\newcommand{\mypause}{\pause}
\newcommand{\defi}[1]{\textbf{#1}}
\newcommand{\red}[1]{\emph{\color{red}#1}}
\newcommand{\pos}{{\textsf{pos}}}
\newcommand{\eff}{{\textsf{eff}}}
\newcommand{\rot}{{\textsf{rot}}}
\newcommand{\veC}{{\textsf{vec}}}
\newcommand{\quat}{{\textsf{quat}}}
\newcommand{\col}{{\textsf{col}}}
\newcommand{\de}[2]{\frac{\partial #1}{\partial #2}}
\newcommand{\target}{{\text{target}}}
\newcommand{\near}{{\text{near}}}
\newcommand{\qfree}{Q_{\text{free}}}
\renewcommand{\vec}{\boldsymbol}
\newcommand{\lft}{\text{left}}
\newcommand{\rgh}{\text{right}}
\newcommand{\prev}{{\text{prev}}}
\newcommand{\TR}[2]{T_{{#1}\to{#2}}}
\newcommand{\RO}[2]{R_{{#1}\to{#2}}}
\newcommand{\liter}{\helvetica{8}{1.1}{m}{n}\parskip 1ex}
\newcommand{\Fc}{\color{green}F}
\newcommand{\muc}{\color{blue}\mu}
\newcommand{\Astar}{A$^*$}

\newcommand{\eqannot}[1]{\tag*{\footnotesize #1}}

\newcommand{\adec}{\r_\a^-}
\newcommand{\ainc}{\r_\a^+}
\newcommand{\ldec}{\r_\l^-}
\newcommand{\linc}{\r_\l^+}
\newcommand{\minc}{\r_\m^+}
\newcommand{\mdec}{\r_\m^-}
\newcommand{\step}{{\delta}}
\newcommand{\lsstop}{\r_{\text{ls}}}
\newcommand{\stepmax}{\step_{\text{max}}}
\newcommand{\Min}{\text{min}}
\newcommand{\Max}{\text{max}}
\DeclareMathOperator*{\argmin}{arg\,min}
\DeclareMathOperator*{\argmax}{arg\,max}
\DeclareMathOperator*{\EEE}{\mathbb{E}}

\definecolor{amaranth}{rgb}{0.9, 0.17, 0.31}
\definecolor{BLUE}{rgb}{0,0,1}
\definecolor{algblue}{HTML}{125DA6} 
\definecolor{linkblue}{RGB}{34,121,181}
\newcommand{\rred}[1]{\textcolor{amaranth}{#1}}
\newcommand{\bblue}[1]{\textcolor{algblue}{#1}}
\newcommand{\oorange}[1]{\textcolor{orange}{#1}}

\newcommand{\wrt}{w.r.t.\@\xspace}
\newcommand{\iid}{i.i.d.\@\xspace}
\newcommand{\cma}{\Delta \vx_{\textsc{cma}}}

\usepackage{mdframed}
\usepackage{etoc}

\definecolor{darkred}{RGB}{100,0,0}
\definecolor{darkblue}{RGB}{0,0,100}
\definecolor{darkgreen}{RGB}{0,75,0}

\hypersetup{citecolor=darkgreen}
\hypersetup{linkcolor=darkgreen}
\hypersetup{urlcolor=darkgreen}

\declaretheoremstyle[
    spaceabove    = \parsep,
    spacebelow    = \parsep,
    bodyfont      = \normalfont\itshape,
]{theoremsty}

\declaretheoremstyle[
    spaceabove    = \parsep,
    spacebelow    = \parsep,
    bodyfont      = \normalfont,
]{normalsty}

\mdfdefinestyle{coloredstyle}{%
    leftline          = false,%
    rightline         = false,%
    topline           = false,%
    bottomline        = false,%
    backgroundcolor   = gray!10,%
    innertopmargin    = 1.5ex,%
    innerbottommargin = 0.7ex,%
    innerleftmargin   = 1.ex,%
    innerrightmargin  = 1.ex,%
    skipabove         = \topskip,%
    skipbelow         = \parskip%
}

\declaretheorem[name=Theorem, style=theoremsty, mdframed={style = coloredstyle}]{theorem}
\declaretheorem[name=Proposition, style=theoremsty, mdframed={style = coloredstyle}]{proposition}
\declaretheorem[name=Corollary, style=theoremsty, mdframed={style = coloredstyle}]{corollary}
\declaretheorem[name=Lemma, style=theoremsty, mdframed={style = coloredstyle}]{lemma}
\declaretheorem[name=Conjecture, style=theoremsty, mdframed={style = coloredstyle}]{conjecture}

\declaretheorem[name=Theorem, style=theoremsty, mdframed={style = coloredstyle}]{theorem*}
\declaretheorem[name=Proposition, style=theoremsty, mdframed={style = coloredstyle}]{proposition*}
\declaretheorem[name=Corollary, style=theoremsty, mdframed={style = coloredstyle}]{corollary*}
\declaretheorem[name=Lemma, style=theoremsty, mdframed={style = coloredstyle}]{lemma*}
\declaretheorem[name=Conjecture, style=theoremsty, mdframed={style = coloredstyle}]{conjecture*}

\declaretheorem[name=Theorem, style=theoremsty, mdframed={style = coloredstyle}, numberwithin=section]{theorem_ap}
\declaretheorem[name=Proposition, style=theoremsty, mdframed={style = coloredstyle}, numberwithin=section]{proposition_ap}
\declaretheorem[name=Corollary, style=theoremsty, mdframed={style = coloredstyle}, numberwithin=section]{corollary_ap}
\declaretheorem[name=Lemma, style=theoremsty, mdframed={style = coloredstyle}, numberwithin=section]{lemma_ap}
\declaretheorem[name=Conjecture, style=theoremsty, mdframed={style = coloredstyle}, numberwithin=section]{conjecture_ap}

\declaretheorem[name=Remark, style=normalsty]{remark}
\declaretheorem[name=Definition, style=normalsty, mdframed={style = coloredstyle}]{definition}
\declaretheorem[name=Assumption, style=normalsty, mdframed={style = coloredstyle}]{assumption}
\declaretheorem[name=Example, style=normalsty, mdframed={style = coloredstyle}]{example}

\declaretheorem[name=Remark, style=normalsty, numbered=no]{remark*}
\declaretheorem[name=Definition, style=normalsty, mdframed={style = coloredstyle}, numbered=no]{definition*}
\declaretheorem[name=Assumption, style=normalsty, mdframed={style = coloredstyle}, numbered=no]{assumption*}

\crefname{assumption}{assumption}{assumptions}

\etocframedstyle[1]{\textbf{\textsc{Table of Contents}}}
\etocsettocdepth{3}

\hypersetup{
  colorlinks,
  linkcolor=black,
  citecolor=black,
  urlcolor={plotblue}
}
\newcommand{\tabpm}[2]{\ensuremath{#1\,\scriptstyle\pm #2}}
\title{Manifold Sampling via Entropy Maximization}

\author{%
Cornelius V.~Braun$^1$\thanks{Equal contribution; Authors in alphabetical order.}
\qquad
Tilman Burghoff$\,^1$\footnotemark[1]
\qquad
Marc Toussaint$^{1,2}$\\
  $^1$Technische Universität Berlin\\
  $^2$Robotics Institute Germany\\
\texttt{\{braun,t.burghoff,toussaint\}@tu-berlin.de} 
}

\begin{document}

\maketitle
\begin{abstract}
    Sampling from constrained distributions has a wide range of applications, including in Bayesian optimization and robotics.
    Prior work establishes convergence and feasibility guarantees for constrained sampling, but assumes that the feasible set is connected.
    However, in practice, the feasible set often decomposes into multiple disconnected components, which makes efficient sampling under constraints challenging.
    In this paper, we propose MAnifold Sampling via Entropy Maximization (MASEM) for uniform sampling on a manifold with an unknown number of disconnected components, implicitly defined by smooth equality and inequality constraints.
    The presented method uses a resampling scheme to maximize the entropy of the empirical distribution based on k-nearest neighbor density estimation.
    We show that, in the mean field, MASEM decreases the KL-divergence between the empirical distribution and the maximum-entropy target exponentially in the number of resampling steps.
    We instantiate MASEM with multiple local samplers and demonstrate its versatility and efficiency on synthetic and robotics-based benchmarks.
    MASEM enables fast and scalable mixing across a range of constrained sampling problems, improving over alternatives by an order of magnitude in Sinkhorn distance with competitive runtime. \href{https://lis-tu-berlin.github.io/masem}{\textcolor{plotblue}{Link to code}}.
\end{abstract}

\section{Introduction}\label{sec:intro}
Sampling from constrained distributions is a fundamental problem in machine learning, with applications including Bayesian inference~\citep{zhang2022sampling} and molecular design~\citep{rapaport2004geometrically} as well as robotics and trajectory optimization~\citep{kingston2018samplingbased}. 
In particular, many tasks in robotics require sampling kinematically feasible states or trajectories, for instance, to generate data for behavior cloning or to sample reset states for reinforcement learning~\citep{power2023sampling, florensa2017reverse, toussaint2026constrained, cobobriesewitz2026stabilityguided}.
The feasible set is usually given implicitly by constraint functions, which can only be evaluated point-wise. 
This makes sampling from it a challenging problem, as we have no prior knowledge about many of its properties, like the number of connected components.

When no constraints are present, the dominant sampling paradigm is built around \ac{mcmc} methods such as Metropolis–Hastings, Langevin dynamics, and Hamiltonian Monte Carlo~\citep{andrieu2003introduction}.
To handle constraints, these approaches are typically combined with projections or landing mechanisms that ensure samples remain feasible~\citep{zhang2022sampling, jeon2025fast, zappa2018monte, lelievre2019hybrid, kaufman1998direction}.
However, all such methods rely on local \ac{mcmc} kernels, which limits them to moves within a single connected region of the feasible set, and as a result they struggle with infeasibility barriers. 
As a result, even if such kernels mix well locally, they cannot determine how much probability mass to assign to each component (as illustrated in \Cref{fig:toy_spheres}). 
This makes them ill-suited in many practical scenarios, where such disconnected feasible sets arise naturally~\citep{le2023accelerating, li2024learning, maus2023discovering}.

We tackle this problem by maximizing sample entropy across the feasible set.
This provides empirical coverage of the feasible set, as the samples are distributed approximately uniformly.
It enables estimation of component volumes and may serve as unbiased initialization for sampling from arbitrary distributions across components.
The key idea is to combine local $k$-nearest-neighbor density estimates \citep{loftsgaarden1965nonparametric} with Sequential Monte Carlo (SMC)-inspired resampling: particles in underrepresented components receive higher importance weights, encouraging redistribution of mass toward a uniform allocation. 
Crucially, the resampling step is modular and can be composed with generic constrained MCMC samplers.
Adding this method to a constrained sampler is minimally invasive in the sense that it only incurs a small runtime overhead and does not degrade performance on compact and connected manifolds.
Our main contributions are as follows:
\begin{itemize}[parsep=1pt, before=\setlength{\parskip}{0pt},leftmargin=2em,itemsep=0pt]
    \item We introduce \textbf{Ma}nifold \textbf{S}ampling via \textbf{E}ntropy \textbf{M}aximization (MASEM),
    an algorithm for uniform sampling on manifolds with disconnected components.
    \item We show that, under mild assumptions, the induced resampling operator contracts the KL-divergence to the uniform distribution at a geometric rate of $(1 - \tau / p)$ per iteration, where $\tau$ is a temperature parameter and $p$ is the intrinsic dimension of the manifold. 
    \item We validate the approach empirically on synthetic and robotics benchmarks and compare performance across different local samplers.
\end{itemize}

\begin{figure}[t]
    \centering
    \includegraphics[width=.95\linewidth, trim={0.25cm 0.2cm 0.2cm .2cm},clip]{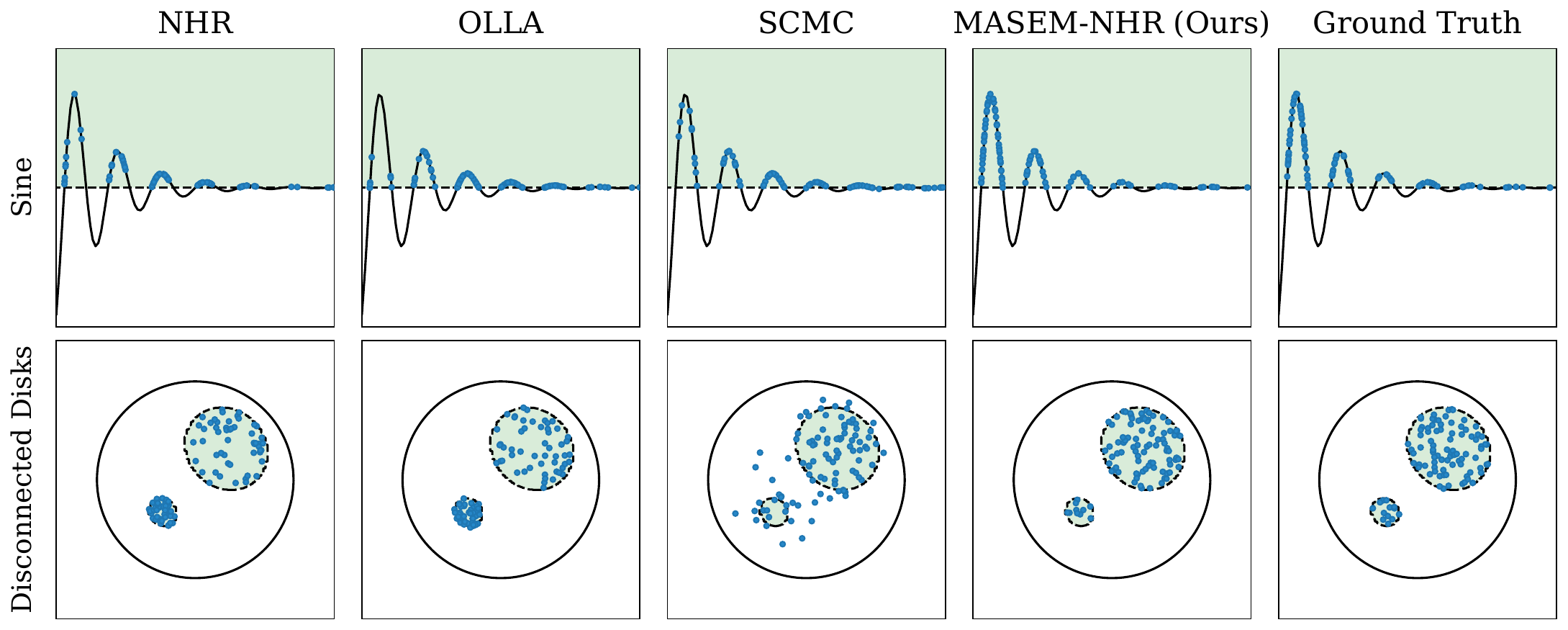}
    \caption{
    Scatter plots of 100 samples from different samplers on synthetic benchmarks. Black lines show equality constraints, and green shaded areas mark the feasible region(s) by inequality constraints. The bottom row depicts 3d samples on spheres projected into 2d. 
    NHR and OLLA fail to allocate correct sample masses: the larger disk and the leftmost peak of the sine are undersampled.
    }
    \label{fig:toy_spheres}
\end{figure}
\section{Related Work}\label{sec:rel}
\textbf{Constrained Sampling.}
Sampling under constraints is typically approached by adapting classical MCMC methods to ensure feasibility, for example by projection, reflection, or annealing techniques.
Soft-constraint and annealing approaches such as sequentially constrained Monte Carlo (SCMC)~\citep{golchi2016sequentially} improve global exploration by relaxing the constraints, but may produce infeasible samples (see \Cref{fig:toy_spheres}) and can still collapse to local constraint minima as the penalty is tightened.
In contrast, our method enables mixing across infeasibility barriers and exact constraint feasibility at convergence.
For manifold-constrained targets, constrained MCMC~\citep{zappa2018monte, lelievre2012langevin, brubaker2012family, byrne2013geodesic, bubeck2018sampling}, barrier-based~\citep{kook2022sampling} and projected Hamiltonian methods~\citep{lelievre2019hybrid}, landing-based approaches~\citep{jeon2025fast, zhang2022sampling}, constrained SVGD~\citep{zhang2022sampling, power2024constrained} and non-linear Hit-and-Run (NHR)~\citep{toussaint2024nlp} operate directly on or near the feasible set and can maintain exact or approximate feasibility.
Despite their differences, these methods share the common structural limitation that they rely on local transition kernels and therefore cannot reliably allocate probability mass across disconnected components.
Our approach is complementary to this line of work: rather than introducing a new local kernel, it adds a resampling mechanism that can be combined with existing constrained samplers.

\textbf{Multimodal Sampling.} 
Multimodal (unconstrained) targets where energy barriers separate distant distributional modes is a long-standing challenge in sampling~\cite{latuszyski2026mcmc}.
Despite the absence of constraints, this problem shares structural similarities with constraint sampling, as the main challenge lies in moving samples globally.
A popular approach to bridging modes in multimodal distributions is \textit{target relaxation} such that the separating regions acquire non-negligible probability mass.
Notably, parallel tempering samples multiple relaxed targets at different temperatures and exchange states across these copies~\citep{swendsen1986replica}, while simulated annealing moves samples between relaxed targets sequentially to ensure sample movements between modes~\citep{del2006sequential, neal2001annealed}.
Further, diffusion-based jump strategies have recently shown high efficacy at facilitating sample mixing on challenging problems.
Specifically, Diffusive Gibbs Sampling (DiGS) convolves the target with a Gaussian kernel and alternates sample noising and denoising in a Metropolis-within-Gibbs scheme~\citep{pmlr-v235-chen24be}, which is extended to discrete targets in HiSS via a logistic bridging kernel~\citep{pmlr-v300-mohanty26a}.
While relaxation is highly effective for global moves, it is fundamentally more difficult when the support region is defined by constraints.
For one, local optima of constraint functions may induced biased samples. 
Further, for equality constraints smoothing replaces the singular target with a full-dimensional distribution concentrated in a neighborhood around it, and moving particles from ambient space onto the manifold is challenging. 
Our method addresses this tension by enabling moves across infeasible barriers via entropy-based resampling while retaining exact constraint feasibility at convergence.

\section{Preliminaries \& Notation}\label{sec:prelim}
We consider sampling uniformly from a feasible set
\begin{equation}
    \Sigma \coloneq \{x \in \mathbb{R}^d \mid h(x)=0,\ g(x)\le 0\}\subseteq\RRR^d,
\end{equation}
where $h\colon\mathbb{R}^d\to\mathbb{R}^m$ and $g\colon\mathbb{R}^d\to\mathbb{R}^\ell$ are differentiable equality and inequality constraint functions, which can only be evaluated point-wise and have Lipschitz-continuous derivatives.
We assume that the feasible set $\Sigma$ is bounded, for example by introducing the additional constraint $\lVert x\rVert\leq R$ for some large $R> 0$.
Under this assumption, the feasible set and its connected components are compact.
Furthermore, we impose the widely used \acf{licq}.
\begin{assumption}[\acl{licq} \citep{aragon2019nonlinear}]\label{ass:licq}
    For each point $x\in\Sigma$, the set of gradients of active constraints
    \begin{equation*}
    \{\nabla h_j(x)\} \cup \{\nabla g_i(x)\mid g_i(x)=0\}
    \end{equation*}
    is linearly independent.
\end{assumption}
This assumption ensures that the feasible set $\Sigma$ is a Riemannian manifold with corners and intrinsic dimension $p=d-m$~\citep[Lemma 3.1.12]{jongen2000nonlinear}.
Let $\mu_\Sigma$ denote the induced Hausdorff measure and $S=\mu_\Sigma(\Sigma)$ its volume.

We consider the general setting that the manifold decomposes into a finite number of disconnected components $\Sigma = \Sigma_1 \sqcup \cdots \sqcup \Sigma_C$.
Importantly, however, we do \emph{not} assume knowledge of the number $C$ of components or any of their structural properties, as they are implicitly defined via $h$ and $g$.
Due to \licq, each component has measure $\sigma_c = \mu_\Sigma(\Sigma_c)>0$.
Note that if $p=d$, this problem reduces to sampling from the unnormalized density $\hat \rho(x)=1_\Sigma(x)$, where $1_\Sigma(x)$ is the indicator function of $\Sigma$. 
Our target is to maximize the entropy of samples across the constrained set.
\begin{lemma}[Informally]\label{lemma:ent_kl}
    The uniform density is the unique maximizer of the entropy on $\Sigma$.
\end{lemma}
We formalize this in \Cref{app:proofs}.
Our method maximizes the entropy by estimating the local density by $k$-nearest neighbor distances~\citep{loftsgaarden1965nonparametric}. 
Let samples $x_1,\dots,x_N\in\Sigma$ be samples drawn from an underlying distribution parameterized by a density $\rho$. 
We define $\varepsilon_{i,k}$ as the distance from $x_i$ to its $k$-th nearest neighbor among $\{x_j\}_{j\ne i}$.
The $k$-nearest-neighbor density estimator at $x_i$ is
\begin{equation}\label{eq:knn_density}
    \hat{\rho}(x_i) = \frac{k}{NV_p\varepsilon_{i,k}^p},
\end{equation}
where $V_p$ is the volume of the unit ball in $\mathbb{R}^p$.
Consistency of this estimator, meaning $\hat{\rho}(x_i) \to \rho(x_i)$ in probability as $N\to\infty$ with $k\to\infty$, $k/N\to 0$, is a classical result for interior points of the support~\citep{biau2015lectures}.

\section{Manifold Sampling via Entropy Maximization}\label{sec:alg}
We now introduce \emph{Manifold Sampling via Entropy Maximization} (MASEM), a framework for uniform sampling on a constrained set $\Sigma$ as defined in \Cref{sec:prelim}.
Our goal is to construct a particle system whose empirical distribution approximates the uniform density $u_\Sigma$ on $\Sigma$.
Given an existing strong local sampler that preserves feasibility, the main challenge for constrained uniform sampling is \textit{global} mass allocation.
As $\Sigma$ is only defined implicitly via constraints, this is a hard problem, since the measures of components or even their number are not known a priori, so local samplers cannot calibrate probability mass across disconnected components efficiently.
Thus, MASEM introduces a mechanism that redistributes particles across components to ensure that each component receives mass proportional to its measure.

The design of MASEM is intentionally lightweight.
Its key mechanism is an importance-resampling step~\citep{andrieu2003introduction} that moves particles towards the maximum entropy target. This target is the uniform density $p=u_\Sigma$ (\Cref{lemma:ent_kl}). 
Let $q$ denote the current particle density on $\Sigma$. 
Since the target density is uniform, importance-resampling with weights $w_i\propto{p(x_i)}/{q(x_i)}$ reduces to resampling according to $w_i\propto q(x)^{-1}$. 
While the current density itself is unknown, we can estimate it using the $k$-nearest neighbor estimator \eqref{eq:knn_density}.
This leads to $q(x_i)\propto \varepsilon_{i,k}^{-p}$ in probability in the large particle limit, where $\varepsilon_{i,k}$ denotes the distance of sample $x_i$ to its $k$th neighbor.
Replacing the potentially unknown dimension $p$ by a temperature parameter $\tau$ yields
\begin{equation}\label{eq:ent_weights}
    \bar w_i = \varepsilon_{i,k}^{\,\tau}, \qquad 
    w_i = \frac{\bar w_i}{\sum_j \bar w_j}\,.
\end{equation}
Thus, particles in low-density regions have larger $k$-NN radii and receive larger resampling weights, encouraging redistribution toward underrepresented connected components.
The temperature $\tau$ controls the aggressiveness of the resampling step. 
When $\tau$ is small, the update is conservative and moves mass more slowly.
Larger values of $\tau$ accelerate mass redistribution but can amplify errors in the $k$-NN density estimate when the number of particles or mixing steps is insufficient.

\begin{algorithm}[t]
    \caption{Entropy-based uniform sampling on disconnected manifolds}
    \label{alg:ours}
    \textbf{Input:}
    Initialization scale $\sigma^2$, number of samples $N$, 
    number of iterations $T$, neighborhood parameter $k$, temperature $\tau$, rejuvenation kernel $K$, rejuvenation step number $M$, constraint functions $h, g$. 
    
    \begin{algorithmic}[1]
        \State \textit{// Initialization}
        \State Sample $N$ particles $x^{i}_0\sim \mathcal{N}(0,\sigma^2)$
        \State Project the particles onto $\Sigma$ (e.g., with Gauss-Newton method)
        \State Apply a manifold-constrained rejuvenation kernel $K$
        \State \textit{// Iterative entropy maximization}
        \For{$t=1,\dots,T$}
            \State Compute weights $w_t^{i}$ from the entropy-based rule in \eqref{eq:ent_weights}
            \State Resample particles according to the normalized weights $w_t^{i}$
            \State Apply a manifold-constrained rejuvenation kernel $K$ for $M$ steps
        \EndFor
        \State \Return final particle set $X^{(T)}$
    \end{algorithmic}
\end{algorithm}

MASEM, described in pseudocode in \Cref{alg:ours}, works by first initializing feasible particles and then iterating resampling followed by manifold sampling rejuvenation. 
We initialize the particles by random \iid sampling followed by projection  onto the feasible set using Gauss-Newton steps on the squared slack
\begin{equation}\label{eq:sqr-slack}
    \operatorname{Slack}(x) = \frac{1}{2}\bigl(\lVert g(x)_+\rVert^2 + \lVert h(x)\rVert^2\bigr)
\end{equation} to ensure approximate constraint feasibility.
After initialization, we apply a feasibility-preserving kernel $K$ for rejuvenation.
This ensures that the particles are approximately uniformly distributed within components.
Then, at each iteration, particles are resampled using the entropy-based resampling weights \eqref{eq:ent_weights}.
After resampling, we apply $K$ again to obtain mixed samples within each component.
As local kernels are only approximately feasible in practice, we use a slack-penalized variant of the resampling weights to avoid over-replicating particles with large constraint violations in our experiments.
These practical modifications are described in \Cref{app:impl-details}.
%

\section{Theoretical Analysis}\label{sec:analysis}
Our theoretical guarantees address the global mass-redistribution problem, which cannot be solved by existing local samplers. 
We analyze \Cref{alg:ours} in the mean-field regime, where the $N \to \infty$ limit reduces the particle system to a deterministic evolution on the simplex of component weights. 
This parallels the idealized-limit viewpoint taken in related work on particle-based constrained sampling~\citep{zhang2022sampling, jeon2025fast}.
Our main result (\cref{thm:main}) shows that a single resampling step contracts the KL divergence to the uniform distribution on $\Sigma$ at the rate $(1-\tau/p)^2$, leading to exponential convergence.
All proofs for this section can be found in \Cref{app:proofs}. 

\subsection{Preliminaries}
We need two prerequisites for our main result: Each component must be initially populated, and the sampler has to sample accurately within components. 
We begin by showing that the initialization step of \Cref{alg:ours} populates all components for $N$ large enough.
%
\begin{lemmaE}[Component coverage at initialization][]\label{lem:inital-populated}
Suppose $h, g$ have Lipschitz-continuous derivatives and satisfy \licq, and that the feasible set $\Sigma$ is bounded and decomposes into a finite number of connected components. We draw $N$ points $x_i\sim\mathcal{N}(0,\sigma^2)$ for arbitrary $\sigma>0$
and apply any descent method on the squared slack \eqref{eq:sqr-slack} satisfying sufficient decrease. Then $\mathbb{P}(\text{every component is hit}) \to\nolinebreak1$ as $N \to \infty$.
\end{lemmaE}
\begin{proof}[Proof sketch]
Take $x^* \in \Sigma_c$ arbitrary.
The local error bound $d(x, \Sigma) \leq \kappa \, \text{Slack}(x)^{1/2}$ holds for some $\kappa > 0$ due to \licq~\citep{andreani2025primaldual}. 
Since $\Sigma$ is compact and has finitely many components, the inter-component distance $\delta \coloneq \min_{i \neq j} d(\Sigma_i, \Sigma_j)$ is strictly positive~\citep[Ch.~XI 4.4]{dugundji1966topology}.
Taking $L\coloneqq\{\text{Slack}(x)\leq (\delta / (3 \kappa))^2\}$ ensures that all $x\in L\cap U_{\delta / 3}(\Sigma_c)$ converge to $\Sigma_c$ under any descent method satisfying sufficient decrease~\citep{karimi2016linear}.
Therefore, the basin of attraction has positive Lebesgue measure, and standard coverage arguments yield the claim.
\end{proof}
\begin{proofE}
    Consider gradient descent on the squared constraint violations
    \begin{equation*}
        \text{Slack}(x) = \frac12 \left(\lVert g(x)_+\rVert^2 + \lVert h\rVert^2\right)\,.
    \end{equation*} 
    We will show that each component $\Sigma_c$ admits a basin of attraction $A_c$ under gradient descent on $\operatorname{Slack}(x)$ with $\mu(A_c)>0$.
    Then, drawing \iid from an uncorrelated Gaussian results in each set of positive measure having a positive probability of being hit. Let the probability of hitting $A_c$ be $p_c>0$.
    The probability that no particle among $N$ \iid draws lands in $A_c$ is $(1 - p_c)^N \to 0$, where $M_0$ is the volume of the sampling region. 
    A union bound over the finitely many components gives the claim.


    It is left to show that each component admits a basin of attraction of positive measure. 
    To that end, we show that for each component there exists a nonempty open set $U_c$, such that gradient descent on points in this set ends up in $\Sigma_c$. Then $U_c\subseteq A_c$ and therefore $\mu(A_c)>0$.

    Fix $\Sigma_c$ and $x^* \in \Sigma_c$. 
    By LICQ, the squared slack $\operatorname{Slack}(x)$ fulfills the error bound \citep{andreani2025primaldual}. That means there exists $\varepsilon, \kappa > 0$ such that
    \begin{equation}\label{eq:err_bound}
        d(x, \Sigma) \leq \kappa \, \operatorname{Slack}(x)^{1/2}
        \quad\text{for all } x \in U_\varepsilon(x^*)\,.
    \end{equation}
    Since the components are compact, there exists \begin{equation*}
        \delta \coloneq \min_{i \neq j} d(\Sigma_i, \Sigma_j) > 0
    \end{equation*}
    by standard topological arguments \citep[Ch.~XI 4.4]{dugundji1966topology}.
    We choose $\alpha > 0$ such that $\sqrt{\alpha} \leq \delta/(3\kappa)$, and define $L \coloneq \{x \in U_\varepsilon(x^*) \mid \text{Slack}(x) < \alpha\}$.
    For any $x \in L$,
    \begin{equation*}
        d(x, \Sigma) \leq \kappa \operatorname{Slack}(x)^{1/2} < \kappa\sqrt{\alpha}
        \leq \delta/3.
    \end{equation*}

    By \citet{karimi2016linear}, gradient descent from any $x_0 \in L$ converges to $\Sigma$, and, since \mbox{$\text{Slack}(x_{k+1})\nolinebreak\leq \text{Slack}(x_k)$}, iterates remain in $L$.
    Therefore any $x_0 \in U_c \coloneq L \cap U_{\delta/3}(\Sigma_c)$ converges into $\Sigma_c$.
    The set $U_c$ is open and nonempty, hence $\mu(A_c) \geq \mu(U_c) > 0$, which completes the proof.
\end{proofE}

%

To consider component-level convergence properties, we assume that the local kernel perfectly mixes chains within each component. This assumption decouples the analysis of the resampling in \Cref{alg:ours} from any specific local sampler.
\begin{assumption}[Mixing Rejuvenation]\label{ass:mix-rejuv}
    The rejuvenation kernel $K$ is $u_c$-invariant on each $\Sigma_c$ (i.e., $u_c K = u_c$, where $u_c = \sigma_c^{-1}\mathbf{1}_{\Sigma_c}$), and mixes sufficiently, such that after application the particles within $\Sigma_c$ are approximately i.i.d.\ draws from~$u_c$.
\end{assumption}
This assumption ensures that chains stay within their component during mixing. In practice, if components are close together, the chain might be able to jump between components. For the sake of our analysis, we treat these clusters of components as one \enquote{meta-component}, since, given enough time, the chains will mix uniformly within this cluster.

This assumption allows us to consider component-level weights instead of individual particles. 
Let $\alpha_c \coloneq N_c/N$ denote the fraction of particles in $\Sigma_c$. 
Then the empirical density after rejuvenation takes the mixture form 
$p_\alpha(x) = \alpha_{c(x)}$, 
where $c(x)$ denotes the component containing $x$.
Let $\alpha = (\alpha_1, \dots , \alpha_C )$ be the vector of component weights that induces the distribution $p_\alpha$. 
The sampling problem thus reduces to steering $\alpha \in \Delta^{C-1}$ towards the maximum entropy distribution $\alpha^* \coloneq (\sigma_1/S, \dots, \sigma_C/S)$.

\subsection{Convergence Guarantees}
\label{sec:theory-convergence}
We first derive the deterministic map that governs the evolution of $\alpha$ under resampling, then show that its iterates contract in KL divergence.

\begin{propositionE}[Mean-field resampling map][]
\label{prop:phi-op}
Under Assumption~\ref{ass:mix-rejuv}, the $k$-NN radius of a particle $x_i \in \Sigma_c$ concentrates at $\varepsilon_{i,k} \asymp (k\sigma_c / (N V_p \alpha_c))^{1/p}$ in the limit $N \to \infty$ with $k_N \to \infty$, $k_N/N \to\nolinebreak 0$~\citep[Ch.~2]{biau2015lectures}. 
Consequently, the resampling weights $w_i \propto \varepsilon_{i,k}^\tau$ induce the map
\begin{equation}
  \label{eq:phi-op}
  (\Phi\alpha)_c \coloneq \frac{\alpha_c^{1-\beta}\,(\alpha_c^*)^\beta}{\sum_{j=1}^C \alpha_j^{1-\beta}\,(\alpha_j^*)^\beta},
  \qquad \beta \coloneq \tau / p.
\end{equation}
\end{propositionE}
\begin{proofE}
The mass assigned to component $c$ after normalization is
\begin{equation*}
  (\Phi\alpha)_c
  = \frac{\sum_{x_i \in \Sigma_c} w_i}{\sum_j \sum_{x_k \in \Sigma_j} w_k}
  = \frac{N_c w_c}{\sum_j N_j w_j}
  = \frac{\alpha_c\,(\sigma_c/\alpha_c)^\beta}
         {\sum_j \alpha_j\,(\sigma_j/\alpha_j)^\beta}
  = \frac{\alpha_c^{1-\beta}\sigma_c^\beta}
         {\sum_j \alpha_j^{1-\beta}\sigma_j^\beta},
\end{equation*}
where we used $N_c = \alpha_c N$ and the concentration $w_c \propto (\sigma_c/\alpha_c)^\beta$. 
Substituting $\sigma_c = S\alpha_c^*$ and cancelling $S^\beta$ gives~\eqref{eq:phi-op}.
\end{proofE}

The map~\eqref{eq:phi-op} is a geometric mean of the current distribution $\alpha$ and the target $\alpha^*$.
Its \emph{unique} fixed point on $\Delta^{C-1}$ is $\alpha^*$. 
Thus, iterating $\Phi$ contracts $\alpha$ towards $\alpha^*$, with the parameter $\beta = \tau/p$ controlling the aggressiveness of rebalancing. 
We show that iterating $\Phi$ drives $\alpha$ to $\alpha^*$ at a geometric rate.

\begin{theoremE}\label{thm:main}
    Consider constraints $h, g$ with Lipschitz-continuous derivative fulfilling \licq, which define a bounded constrained set $\Sigma$ that decomposes into a finite number of connected components. 
    Suppose all components have positive measure with respect to the induced Hausdorff measure on $\Sigma$.
    %
    Under \Cref{ass:mix-rejuv} for any $\tau \in (0, p)$ and any schedule $k_N \to \infty$, $k_N/N \to 0$, the mean-field iterates $\alpha_{t+1} = \Phi(\alpha_t)$ of \Cref{alg:ours} satisfy the following bound

    \begin{equation}\label{eq:final-kl-bound}
        D_{\mathrm{KL}}^\Sigma(p_{\alpha_t} \| p_{\alpha^*}) \leq C_0 \left(1 - \frac{\tau}{p}\right)^{2t}
    \end{equation}
    with 
    \begin{equation}\label{eq:C0}
        C_0 = \frac14\left(\max_c \log\frac{\alpha_c^{(0)}}{\alpha_c^*} - \min_c \log\frac{\alpha_c^{(0)}}{\alpha_c^*}\right)^2
    \end{equation}
    for the initial component weights $\alpha_c^{(0)}$ after the projection step of the algorithm.
    In particular, $p_{\alpha_t} \to u_\Sigma$ in KL as $t \to \infty$.
\end{theoremE}
\begin{proof}[Proof sketch]
    We define the relative estimation error $y_{t,c}\coloneqq \tfrac{\alpha_{t,c}}{\alpha^*_c}$ as well as and $z_c \coloneq \log y_{0,c}$ and show using induction that 
    $y_{t,c} = {y_{0,c}^{(1-\beta)^t}} / {\tsum_j \alpha_j^* y_{0,j}^{(1-\beta)^t}}$.
    Writing $a_t \coloneq (1-\beta)^t$, this admits the exponential closed form
    \begin{equation}\label{eq:tilted-family}
      \alpha_{t,c} \coloneq \frac{\alpha_c^* \, e^{a_t z_c}}{Z(a_t)},
      \qquad Z(a) \coloneq \sum_j \alpha_j^* e^{a z_j} = \mathbb{E}_{\alpha^*}[e^{az}],
    \end{equation}
    from which we can see $\alpha_t$ is an exponential tilt of $\alpha^*$, with inverse temperature $a_t$, which shrinks geometrically to zero.
    
    Substituting~\eqref{eq:tilted-family} into the definition of KL gives the decomposition
    \begin{equation*}
        D_{\mathrm{KL}}^\Sigma(p_{\alpha_t} \| p_{\alpha^*}) = a_t \sum_c\alpha_{t,c}\log y_{0,c} - \log Z(a_t),
    \end{equation*}
    so it suffices to upper-bound $\bar z_t$ and lower-bound $\log Z(a_t)$.
    Both bounds follow from properties of the cumulant generating function (CGF) $\Lambda(a) := \log Z(a)$ of $z$ under $\alpha^*$.
    For the lower bound, Jensen's inequality applied to the convex function $e^{a z}$ gives $\log Z(a_t) \geq a_t\sum_c\alpha^*z_c$.
    For the upper bound, we use basic properties of the CGF and Popoviciu's variance inequality~\citep{popoviciu1935equations} to obtain a bound of $\Lambda''(a) = \mathbb{V}_{\nu_a}[z] \leq R^2/4$, where $R \coloneq \max_c \log(\alpha_c^{(0)}/\alpha_c^*)
       - \min_c \log(\alpha_c^{(0)}/\alpha_c^*)$.
    Integrating then gives $\bar z_t \leq \bar z_{\alpha^*} + a_t R^2/4$.
    As the linear terms cancel, we end with 
    $D_{\mathrm{KL}}^\Sigma(p_{\alpha_t} \| p_{\alpha^*}) \leq (R^2/4)(1-\beta)^{2t}$, which is the bound in \eqref{eq:final-kl-bound}.
\end{proof}
\begin{proofE}
    We prove the claim in three steps: (i) derive a closed form for the iterates, (ii) decompose the KL, and (iii) bound each term.

    \textbf{Step 1: Closed form of the iterates.}
    We recall that under Assumption~\ref{ass:mix-rejuv}, the $k$-NN radius of a particle $x_i \in \Sigma_c$ concentrates at $\varepsilon_{i,k} \asymp (k\sigma_c / (N V_p \alpha_c))^{1/p}$ in the limit $N \to \infty$ with $k_N \to \infty$, $k_N/N \to\nolinebreak 0$~\citep[Ch.~2]{biau2015lectures}. 
    Hence, our weights concentrate as \begin{equation*}
        w_i^t = \varepsilon_{i,k}^{\tau} \asymp \left(\frac{k}{NV_p}\frac{\sigma_c}{\alpha_c}\right)^{\tau/p} \propto \left( \frac{\sigma_c}{\alpha_c} \right)^{\tau / p}\,.
    \end{equation*}

    Define $y_{t,c} \coloneq \alpha_{t,c}/\alpha_c^*$. 
    Substituting $\alpha_{t,c} = y_{t,c}\alpha_c^*$ into~\eqref{eq:phi-op} then yields:
    \begin{align*}
      (\Phi\alpha_t)_c
      = \frac{(y_{t,c}\alpha_c^*)^{1-\beta}(\alpha_c^*)^\beta}
             {\sum_j (y_{t,j}\alpha_j^*)^{1-\beta}(\alpha_j^*)^\beta}
      = \frac{\alpha_c^* y_{t,c}^{1-\beta}}
             {\sum_j \alpha_j^* y_{t,j}^{1-\beta}}.
    \end{align*}
    Using this expression, we can characterize how the gap $y_{t,c}$ changes over time: 
    \begin{equation}\label{eq:recc}
      y_{t+1,c} = \frac{\alpha_{t+1, c}}{\alpha^*_c}= \frac{\cancel{\alpha_c^*} y_{t, c}^{1-\beta}}{\cancel{\alpha_c^*}\sum_j \alpha_j^* y_{t, j}^{1-\beta}} = y_{t,c}^{1-\beta}/Z_t
    \end{equation}
    with $Z_t \coloneq \sum_j \alpha_j^* y_{t,j}^{1-\beta}$ independent of component $c$.
    This implies the following recurrence which we prove by induction:
    \begin{equation}
      \label{eq:induction-closed-form}
      y_{t,c} = \frac{y_{0,c}^{a_t}}{\tilde Z_t},
      \qquad a_t := (1-\beta)^t,
      \qquad \tilde Z_t := \sum_j \alpha_j^* y_{0,j}^{a_t}.
    \end{equation}
    \textit{Base case $t=0$}. We have $a_0=1$ so $\tilde Z_0 = \sum_j \alpha_j^{(0)} = 1$.
    This reduces the statement to $y_{0, c} = y_{0, c}$ which is trivially true.
    
    \textit{Inductive step}. We have
    \begin{align*}
        y_{t+1, c} &= \frac{y_{t, c}^{1-\beta}}{\sum_j \alpha_j^* {y_{t, j}}^{1-\beta}}  \eqannot{By \eqref{eq:recc}}\\
        &= \frac{\left(\frac{y_{0,c}^{a_t}}{Z_t}\right)^{1-\beta}}{\sum_j \alpha_j^* \left(\frac{y_{0,j}^{a_t}}{Z_t}\right)^{1-\beta}}   \eqannot{By \eqref{eq:induction-closed-form}}\\
        &= \frac{\frac{y_{0,c}^{a_t(1-\beta)}}{\cancel{Z_t^{1-\beta}}}}{\frac{1}{\cancel{Z_t^{1-\beta}}}\sum_j \alpha_j^* y_{0,j}^{a_t(1-\beta)}} \nonumber\\
        &= \frac{y_{0,c}^{a_{t+1}}}{\sum_j \alpha_j^* y_{0,j}^{a_{t+1}}},  \eqannot{By $a_t(1-\beta) = (1-\beta)^{t+1} = a_{t+1}$}
    \end{align*}
    which completes the induction.

    Now we can express this in tilted exponential form by writing $z_c := \log y_{0,c}$ and $Z(a) := \sum_j \alpha_j^* e^{az_j}$.
    Thus \eqref{eq:induction-closed-form} becomes
    \begin{equation}
      \label{eq:tilted-family-full}
      \alpha_{t,c} = \frac{\alpha_c^*\,e^{a_t z_c}}{Z(a_t)}.
    \end{equation}
    This expression quantifies how much component $c$ was initially over- or under-represented relative to the target and still is at time $t$.

    \textbf{Step 2: KL decomposition.}
    Using~\eqref{eq:tilted-family-full}, we decompose the KL divergence as follows:
    \begin{equation}\label{eq:kl-decomp}
        D_{\mathrm{KL}}^\Sigma(p_{\alpha_t} \| p_{\alpha^*}) = \sum_c \alpha_{t, c} \log \frac{\alpha_c^* e^{a_t z_c}}{\alpha_c^*}
        = \sum_c \alpha_{t, c} \left( a_t z_c - \log Z(a_t) \right)
        = a_t \bar{z}_t - \log Z(a_t),
    \end{equation}
    where $\bar{z}_t \coloneq \sum_c \alpha_{t, c} z_c$ is the first moment of $z$ under $\alpha^t$ and $Z(a) \coloneq \sum_j \alpha_j^* e^{a z_j}$ is the normalizing constant.

    \textbf{Step 3: Bounds on each term.}
    Note that $\alpha_t$ is itself a tilted distribution from the exponential form $\alpha_c^t = \alpha_c^* e^{a_t z_c} / Z(a_t)$ and that $Z(a)$ is its moment generating function.
    We can therefore study how its mean $\bar{z}_t$ deviates from the untilted mean $\bar{z}_{\alpha^*} = \sum_c \alpha_c^* z_c$ by treating the tilt parameter $a$ as continuous and tracking the mean $\mu(a) \coloneq \sum_c \nu_{a, c} z_c$ as $a$ varies from $0$ (where $\nu_0 = \alpha^*$) to $a_t$ (where $\nu_{a_t} = \alpha_t$).

    \textit{Upper Bound via Popoviciu's Inequality.}
    Define the log-partition function $\Lambda(a) \coloneq \log Z(a)$.
    Differentiation of this cumulant generating function yields the first moment of the tilted distribution:
    \begin{equation}
        \Lambda'(a) = \frac{\sum_c \alpha_c^* z_c e^{a z_c}}{Z(a)} = \sum_c \nu^{(a)}_c z_c = \mathbb{E}_{\nu_{a}}[z] = \mu(a).
    \end{equation}
    Differentiating once more yields the variance thereof~\citep{wainwright2008graphical}:
    \begin{equation}
        \mu'(a) = \Lambda''(a) = \sum_c \nu_{a, c} z_c^2 - \left(\sum_c \nu_{a, c} z_c\right)^2 = \mathbb{V}_{\nu_{a}}[z].
    \end{equation}
    Since $z$ only takes values in $[\min_c z_c, \max_c z_c]$, we can apply Popoviciu's variance inequality~\citep{popoviciu1935equations}, which yields $\mathbb{V}_{\nu_a}(z) \leq R^2 / 4$ for all $a$.
    In the last step we define $R \coloneq \max_c \log(\alpha_c^{(0)}/\alpha_c^*) - \min_c \log(\alpha_c^{(0)}/\alpha_c^*)$ for clarity.
     
    Integrating this bound over time from $0$ to $a_t$ yields a bound on the gap $\mu(a_t) - \mu(0) = \int_0^{a_t} \mu'(a) \, da \leq a_t \frac{R^2}{4}$.
    This establishes a bound for the first term in \eqref{eq:kl-decomp} since 
    \begin{equation}\label{eq:mu-bound}
        a_t \bar{z}_t = a_t\mu(a_t) \leq a_t\mu(0) + a_t^2 R^2/4.
    \end{equation}

    \textit{Lower Bound via Jensen's Inequality.}
    We apply Jensen's inequality directly to the convex exponential function in the log-partition function: 
    \begin{equation*}
        Z(a) = \sum_c \alpha_c^* e^{a z_c} \geq \exp\left( a \sum_c \alpha_c^* z_c \right) = e^{a \bar{z}_{\alpha^*}}.
    \end{equation*}
    Hence, we get
    \begin{equation}
      \label{eq:jensen-bound}
      \log Z(a_t) \geq a_t\,\bar z_{\alpha^*}.
    \end{equation}

    \textit{The KL Bound.}
    Substituting~\eqref{eq:mu-bound} and~\eqref{eq:jensen-bound} into~\eqref{eq:kl-decomp}:
    \begin{align*}
      D_{\mathrm{KL}}^\Sigma(p_{\alpha_t}\|p_{\alpha^*})
      \leq a_t\big(\bar z_{\alpha^*} + a_t R^2/4\big) - a_t\bar z_{\alpha^*}
      = \tfrac{R^2}{4}\,(1-\beta)^{2t}.
    \end{align*}
    Since $p_{\alpha_t}$ and $u_\Sigma$ differ only in their component weights and Assumption~\ref{ass:mix-rejuv} ensures within-component uniformity,
    $D_{\mathrm{KL}}^\Sigma(p_{\alpha_t}\|u_\Sigma) = D_{\mathrm{KL}}^\Sigma(p_{\alpha_t}\|p_{\alpha^*})$, yielding~\eqref{eq:final-kl-bound}.
\end{proofE}

Using this result we derive an asymptotic bound on the worst case number of iterations.
\begin{corollaryE}\label{cor:convergence}
    Assume the conditions of \cref{thm:main} hold. 
    In the worst case, the number of iterations $t$ to reach a KL-divergence $D_{\mathrm{KL}}^\Sigma(p_{\alpha_t} \| p_{\alpha^*}) \leq \varepsilon$ is 
    \begin{equation*}
        t\in \mathcal{O}\left(\frac{\log(1/\varepsilon)+ \log\log N}{\tau}\right)\,.
    \end{equation*}
\end{corollaryE}
\begin{proof}[Proof sketch]
    We show $C_0\in\mathcal{O}\left((\log N)^2\right)$. Bounding \eqref{eq:final-kl-bound} from above by $\varepsilon$ lets us derive
    \begin{equation*}
        t\approx \frac{\log(\varepsilon/C_0)}{2\log\bigl(1-\frac\tau p\bigr)} 
        = \frac{\log(1/\varepsilon)+\log C_0}{-2\log(1-\frac\tau p)} \in \mathcal{O}\left(\frac{\log(1/\varepsilon)+ \log\log N}{\tau}\right)\,.\qedhere
    \end{equation*}
\end{proof}
\begin{proofE}
    We start by providing an upper bound on the initialization constant $C_0$ (\cref{eq:C0}) by considering worst-case initialization.
    We assume using \cref{lem:inital-populated} that each component contains at least one sample, therefore $\alpha_c^{(0)}\in\left[\frac1N,1\right)$. This lets us derive \begin{equation*}
        \max_c \log\frac{\alpha_c^{(0)}}{\alpha_c^*}\leq\max_c \log\frac{1}{\alpha_c^*} = - \min_c \log\alpha_c^*
    \end{equation*} and similarly
    \begin{equation*}
        \min_c \log\frac{\alpha_c^{(0)}}{\alpha_c^*}\geq\min\log\frac{1/N}{\alpha_c^*} = -\log(N)-\max_c\log\alpha_c^*\,.
    \end{equation*}
    Therefore\begin{equation*}
        R =  \max_c \log\frac{\alpha_c^{(0)}}{\alpha_c^*} - \min_c \log\frac{\alpha_c^{(0)}}{\alpha_c^*}\leq \log N + \max_c\log\alpha_c^* - \min_c \log\alpha_c^*
    \end{equation*}
    and, since $R>0$, we know that $C_0$ fulfills
    \begin{equation*}
        C_0 = \frac14 R^2\leq \frac14\left(\log N +\max_c\log \alpha_c^* - \min\log\alpha_c^*\right)^2
    \end{equation*}
    and therefore $C_0\in\mathcal{O}\left((\log N)^2\right)$ since $\alpha^*_c = \sigma_c$ is independent of $N$.
    
    By solving $C_0(1-\frac\tau p)^{2t}=\varepsilon$ for $t$ we then get\begin{equation*}
        t\approx \frac{\log(\varepsilon/C_0)}{2\log\bigl(1-\frac\tau p\bigr)} = \frac{\log(1/\varepsilon)+\log C_0}{-2\log(1-\frac\tau p)}\in\mathcal{O}\left(\frac{\log(1/\varepsilon)+ \log\log N}{\tau}\right)
    \end{equation*}
    since $\beta=\frac\tau p\in (0,1)$ and therefore $\beta\leq-\log(1-\beta)$.
\end{proofE}

\textbf{Component loss} We note that \Cref{thm:main} is only stated in the mean field. 
It might happen that no particles of a particular component are sampled, resulting in that component vanishing. 
The probability of this extinction is  $P_c = (1-(\Phi\alpha)_c)^N\leq\exp\bigl(-N(\Phi\alpha)_c\bigr)$.
Using \Cref{prop:phi-op} we get $(\Phi\alpha)_c\propto\alpha_c^{1-\beta}(\alpha_c^*)^\beta$. 
By \Cref{lem:inital-populated}, each component contains at least one sample, bounding $\alpha_c = N_c/N\geq N^{-1}$. 
Since $\alpha_c^*$ is constant, we can derive an asymptotic upper bound on the extinction-probability as $P_c\in\mathcal{O}\bigl(\exp(-N\alpha_c^{1-\beta})\bigr) = \mathcal{O}\bigl(\exp(-N^\beta)\bigr)$. 
We investigated this effect and found that, in practice, just four chains per component is enough (see \cref{app:mode-loss}).

\section{Experimental Analysis}\label{sec:experiments}
To demonstrate how MASEM improves existing manifold samplers, we consider Non-Linear Hit-\&-Run (NHR)~\citep{toussaint2024nlp} and OLLA~\citep{jeon2025fast} with- and without the proposed resampling logic.
In addition, we compare against Sequentially Constrained MC (SCMC)~\citep{golchi2016sequentially} which anneals soft constraints.
To understand how MASEM compares against explicit component discovery and importance sampling, we implement Cluster-NHR, which clusters samples and computes resampling weights based on the estimated cluster volumes. 
We report the squared Sinkhorn distance and the averaged maximum slack violation. Since SCMC cannot guarantee exact constraint feasibility, we perform an additional projection step before reporting the metrics.
Implementation details are listed in \Cref{app:exp-details} and additional figures can be found in \Cref{app:plots}.




\begin{table}[h]
    \centering
    \small
    \caption{Final $\mathcal{W}_2^2$ (Sinkhorn) distance performance across five seeds. Entries are mean $\pm$ 95\% confidence interval. Bold indicates better than counterpart for t-test with $p<0.05$ and Holm-Bonferroni correction. MASEM-Sampler outperform their counterpart by an order of magnitude on all non-connected problems.}
\resizebox{\linewidth}{!}{%
    \begin{tabular}{lcccc}
    \toprule
    Problem & NHR & MASEM-NHR & OLLA & MASEM-OLLA \\
    \cmidrule(lr){1-1}\cmidrule(lr){2-3}\cmidrule(lr){4-5}
    Connect.\ Disks ($3d$) & $\tabpm{\phantom{000}.06}{\phantom{0}.01}$ & $\mathbf{\tabpm{\phantom{0}.00}{\phantom{0}.00}}$ & $\tabpm{\phantom{000}.04}{\phantom{0}.00}$ & $\mathbf{\tabpm{\phantom{0}.00}{\phantom{0}.00}}$ \\
    Disconn.\ Disks ($3d$) & $\tabpm{\phantom{00}1.98}{\phantom{0}.12}$ & $\mathbf{\tabpm{\phantom{0}.01}{\phantom{0}.01}}$ & $\tabpm{\phantom{00}1.95}{\phantom{0}.11}$ & $\mathbf{\tabpm{\phantom{0}.02}{\phantom{0}.00}}$ \\
    Seven Lobes ($2d$) & $\tabpm{\phantom{000}.37}{\phantom{0}.02}$ & $\mathbf{\tabpm{\phantom{0}.05}{\phantom{0}.02}}$ & $\tabpm{\phantom{000}.89}{\phantom{0}.08}$ & $\mathbf{\tabpm{\phantom{0}.05}{\phantom{0}.02}}$ \\
    Sine ($2d$)        & $\tabpm{100.21}{4.09}$                    & $\mathbf{\tabpm{\phantom{0}.13}{\phantom{0}.06}}$ & $\tabpm{128.38}{3.79}$ & $\mathbf{\tabpm{\phantom{0}.13}{\phantom{0}.07}}$ \\
    Swiss Roll ($2d$) & $\tabpm{\phantom{00}6.08}{\phantom{0}.75}$ & $\mathbf{\tabpm{\phantom{0}.08}{\phantom{0}.02}}$ & $\tabpm{\phantom{0}18.96}{1.15}$ & $\mathbf{\tabpm{\phantom{0}.29}{\phantom{0}.09}}$ \\
    \bottomrule
    \end{tabular}
}
\label{tab:sd_final}
\end{table}

\begin{wraptable}{R}{0.55\textwidth}
\centering
\small
\caption{Final $\mathcal{W}_2^2$ (Sinkhorn) distance and slack across five seeds. Entries are mean $\pm$ 95\% confidence interval. Bold indicates best for t-test with $p<0.05$ and Holm-Bonferroni correction.}
\resizebox{\linewidth}{!}{%
\begin{tabular}{lccc}
\toprule
$\mathcal{W}_2^2$ & SCMC & Cluster-NHR & MASEM-NHR \\
\midrule
C.~Disks & $\tabpm{\phantom{000}.12}{\phantom{00}.01}$ & $\tabpm{\phantom{0}.05}{\phantom{0}.03}$ & $\mathbf{\tabpm{\phantom{0}.00}{\phantom{0}.00}}$ \\
D.~Disks & $\tabpm{\phantom{000}.59}{\phantom{00}.23}$ & $\tabpm{\phantom{0}.19}{\phantom{0}.14}$ & $\mathbf{\tabpm{\phantom{0}.01}{\phantom{0}.01}}$ \\
Seven Lobes & $\tabpm{\phantom{000}.12}{\phantom{00}.06}$ & $\tabpm{\phantom{0}.14}{\phantom{0}.04}$ & $\mathbf{\tabpm{\phantom{0}.05}{\phantom{0}.02}}$ \\
Sine & $\tabpm{126.79}{10.85}$                            & $\tabpm{\phantom{0}.91}{1.04}$ & $\tabpm{\phantom{0}.13}{\phantom{0}.06}$ \\
Swiss Roll & $\tabpm{\phantom{00}1.10}{\phantom{00}.42}$ & $\tabpm{\phantom{0}.42}{\phantom{0}.26}$ & $\mathbf{\tabpm{\phantom{0}.08}{\phantom{0}.02}}$ \\
\midrule
Slack $\times 10^{-3}$ & SCMC & Cluster-NHR & MASEM-NHR \\
\midrule
C.~Disks & $\tabpm{20.46}{\phantom{0}1.12}$ & $\tabpm{\phantom{0}.02}{\phantom{0}.03}$ & $\tabpm{\phantom{0}.01}{\phantom{0}.01}$ \\
D.~Disks & $\tabpm{29.97}{13.61}$ & $\tabpm{\phantom{0}.02}{\phantom{0}.03}$ & $\tabpm{\phantom{0}.01}{\phantom{0}.01}$ \\
Seven Lobes & $\tabpm{\phantom{0}3.00}{\phantom{0}1.89}$ & $\tabpm{\phantom{0}.00}{\phantom{0}.00}$ & $\tabpm{\phantom{0}.00}{\phantom{0}.00}$ \\
Sine & $\tabpm{\phantom{0}9.38}{\phantom{0}1.48}$ & $\mathbf{\tabpm{\phantom{0}.00}{\phantom{0}.00}}$ & $\tabpm{\phantom{0}.01}{\phantom{0}.01}$ \\
Swiss Roll & $\tabpm{\phantom{0}2.77}{\phantom{00}.71}$ & $\tabpm{\phantom{0}.00}{\phantom{0}.00}$ & $\mathbf{\tabpm{\phantom{0}.00}{\phantom{0}.00}}$ \\
\bottomrule
\end{tabular}%
}
\label{tab:sd_final_glob}
\end{wraptable}

\subsection{Synthetic Benchmarks}
We first evaluate sampling on low-dimensional manifolds: (1)-(2) two disks that are embedded on a sphere in $\RRR^3$ and we vary whether the disks are connected or not, (3) the seven lobes density from ~\citet{jeon2025fast} with constant $f$, (4) a sine equality constraint with reducing amplitude cut into disconnected components inequality constraint, and (5) a product of an Archimedean spiral manifold and randomly sampled circles combined with a nonlinear inequality we label swiss roll. 
The benchmarks are visualized in \Cref{fig:2dbenchs}, and we report the specific constraints and ground truth sampling method in \Cref{app:synth-constraints}.
For each problem, we run $2\,000$ independent chains in parallel with a budget of $5\,000$ target evaluations and collect the last sample per chain.

\textbf{Local Kernels Fail Under Disconnected Components.}
In \Cref{fig:toy_spheres} we identify a common failure mode of standard manifold samplers: In the presence of disconnected components, these local methods fail to correctly allocate mass; both methods over-sample the lower disk on the sphere and under sample the left-most arc of the sine.
While the connected disks problem permits the chains to mix well in the single feasible level set, separating the two disks on the manifold leads to an increase in sampling error, as shown in \Cref{tab:sd_final}.

\textbf{Sampling Accuracy \& Constraint Violation of MASEM.}
As shown in \Cref{tab:sd_final}, MASEM-NHR and MASEM-OLLA outperform their counterparts across all problems with a disconnected feasible set by an order of magnitude in the Sinkhorn distance $\mathcal{W}_2^2$.
Further, MASEM minimizes the constraints, matching (or even outperforming) their local counterparts (\Cref{fig:syn_conv}).
We also see in \Cref{tab:sd_final_glob} that MASEM yields superior performance compared to the global baselines.
In particular, we observe that while SCMC effectively yields constraint satisfying samples on the seven lobes problem, it fails to do so in particular on problems with highly non-linear constraints such as the swiss roll problem.
Cluster-NHR, in contrast, achieves low constraint violations but higher sampling error than MASEM on all problems.



\paragraph{Effect of Hyperparameters}
We further examine the influence of the hyperparameters $\tau$, $k$, the number of rejuvenation steps $M$ and the number of samples $N$ on the seven lobes problem.
As shown in \Cref{fig:tau_grid}, our method is robust to variations of $M$ and $N$. However, setting $\tau$ greater than the intrinsic dimension leads to high Sinkhorn divergences to ground truth samples. Increasing the number of rejuvenation steps or increasing $k$ recovers some of the error at high $\tau$. Conversely, using small values of $k$ risks higher Sinkhorn divergences.

We note that MASEM method is relatively robust to the exact choice of these hyperparameters. 
To aid practitioners, we provide a standard recommendation for the parameters. 
Setting \begin{equation}\label{eq:standardparams}
    \tau=0.3,\ M=20,\ \text{and }k=20 
\end{equation} yields good results across all benchmarks, achieving a final $\mathcal{W}_2^2$ in the same order of magnitude as the tuned version, as we show in \Cref{app:ana-paras}.

\begin{figure} 
    \centering   
    \includegraphics[width=.3\linewidth, trim={0cm 0cm 36cm 0cm}, clip]{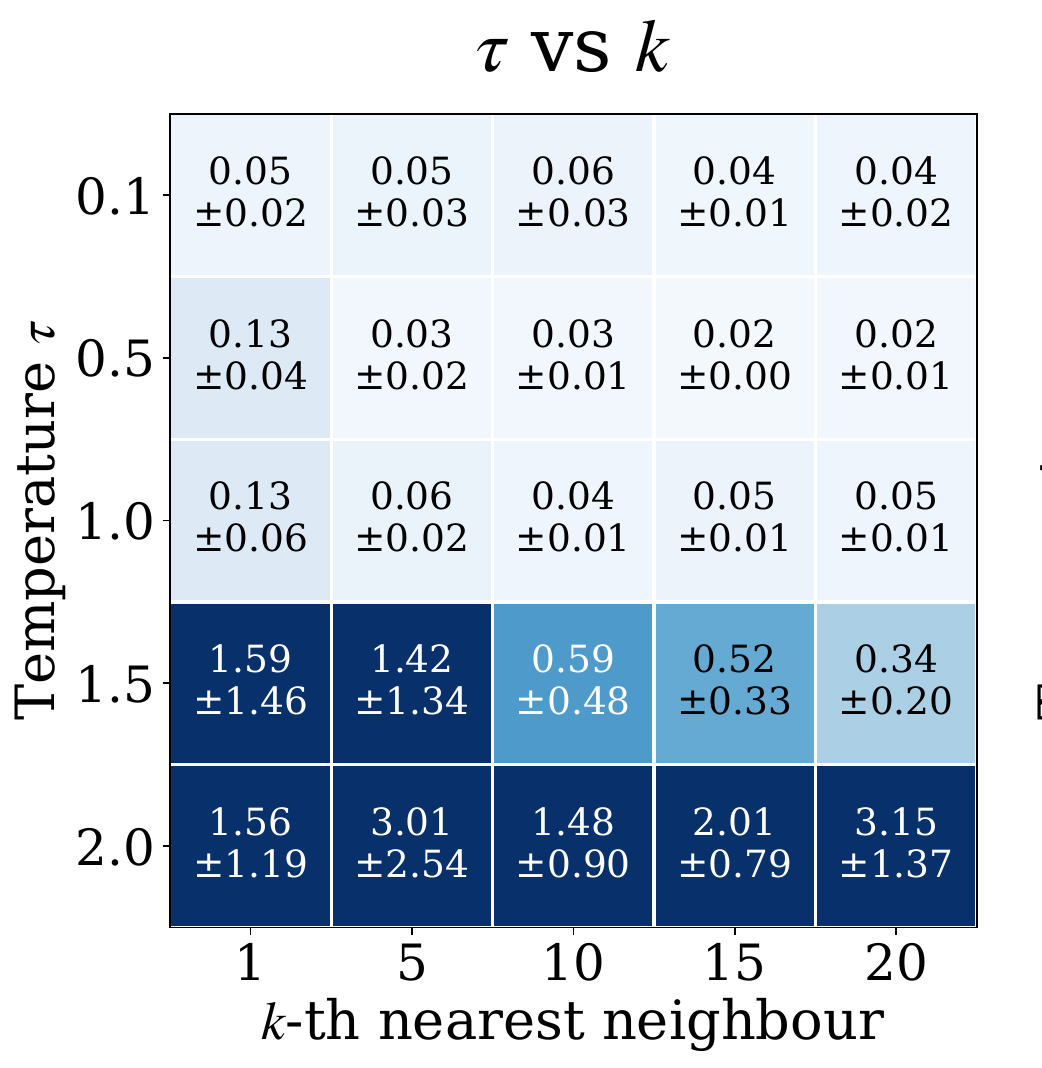}\hfill
    \includegraphics[width=.3\linewidth, trim={17.5cm 0cm 19cm 0cm}, clip]{content/figures/2d/hparams_combined_seven_lobes_sd_proj.pdf}\hfill
    \includegraphics[width=.31\linewidth, trim={35cm 0cm .5cm 0cm}, clip]{content/figures/2d/hparams_combined_seven_lobes_sd_proj.pdf}\\
    \caption{Influence of hyperparameters for MASEM-NHR on the 7 lobes problem. The parameters not in the grid are fixed to parameters $M=10$, $\tau=0.5$, $k=10$ and $N=2000$.}
    \label{fig:tau_grid}
\end{figure}

\subsection{Scaling Under High-Dimensionality and Large Number of Constraints}\label{sec:scaling_exp}
We assess the robustness and scalability of MASEM-NHR and MASEM-OLLA using a synthetic \textit{stress-test problem} that enables explicit control of ambient dimension $d$ as well as the number of equality and inequality constraints and number of disconnected components $(m, l, |C|)$.
In our case, we guarantee different components by sampling a set of $|C|$ disjoint spheres, which we embed into a subspace of $\RRR^d$ by random linear projections defined by the $m$ equality constraints.

\textbf{Scaling Under Increasing Manifold Dimension.}
With fixed numbers of constraints $m=l=5$ we study the effect of an increasing manifold dimension as $d$ increases.
\Cref{fig:scaling_inc_m} shows
that MASEM consistently outperforms NHR and OLLA as well as Cluster-NHR by an order of magnitude.
In addition, we note that the penalty-based approach of SCMC deteriorates strongly as $d$ increases.
Further results in \Cref{fig:scaling_conv} show a similar trend for the KL divergence while the maximum constraint violation stays below $0.1$ for MASEM.

\textbf{Scaling Under Number of Constraints.}
\Cref{fig:scaling_inc_dim} illustrates the performances as the ambient dimension $d$ increases while the manifold dimension is fixed to 2.
We see that the gap between MASEM and its counterparts remains constant across dimensions while the $\mathcal{W}_2^2$ distance only slightly increases with $d$.
The wallclock times are competitive compared to the baselines that evaluate the constraint Jacobian (OLLA, NHR, Cluster-NHR).
This shows that our method adds only little overhead in practice.
Additional plots for the KL divergence and constraint violation in \Cref{fig:scaling_conv} further support these results.


\begin{figure}[t]
    \centering
    \includegraphics[width=.8\linewidth]{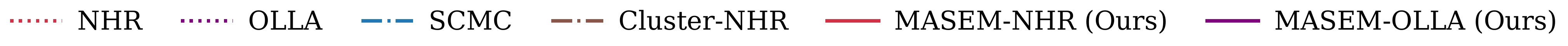}\\
    \begin{subfigure}{0.45\textwidth}
        \centering
        \includegraphics[width=.5\linewidth]{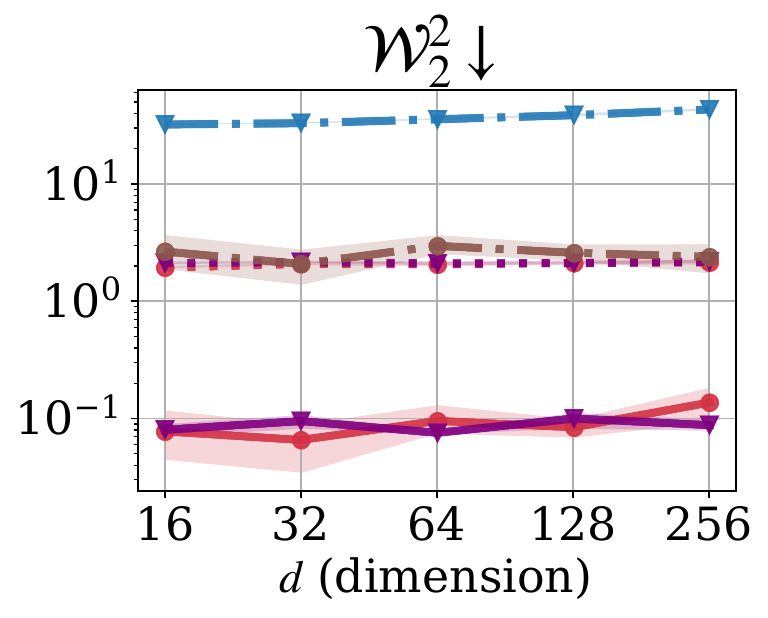}\hfill
        \includegraphics[width=.5\linewidth]{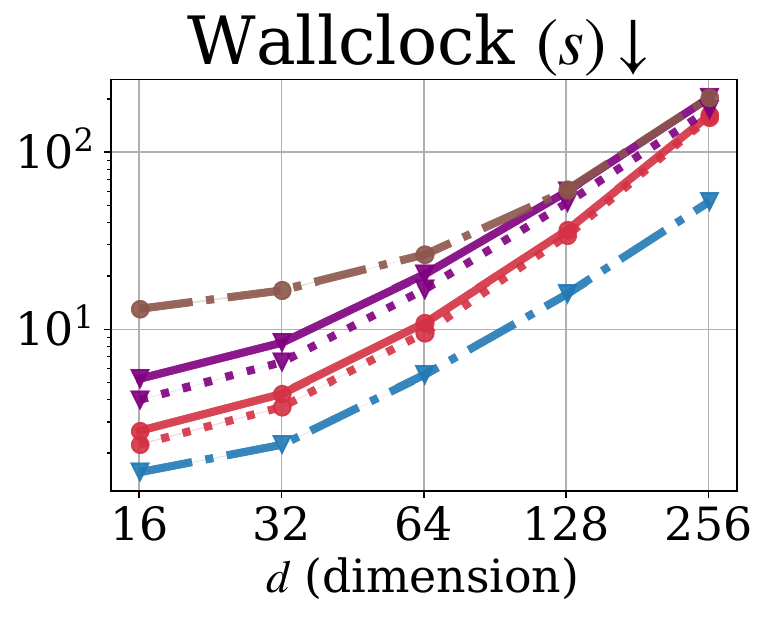}
        \caption{Increasing ambient and fixed manifold dimension with increasing $m=d-3$.}
        \label{fig:scaling_inc_dim}
    \end{subfigure}
    \hfill
    \begin{subfigure}{0.45\textwidth}
        \centering
        \includegraphics[width=.5\linewidth]{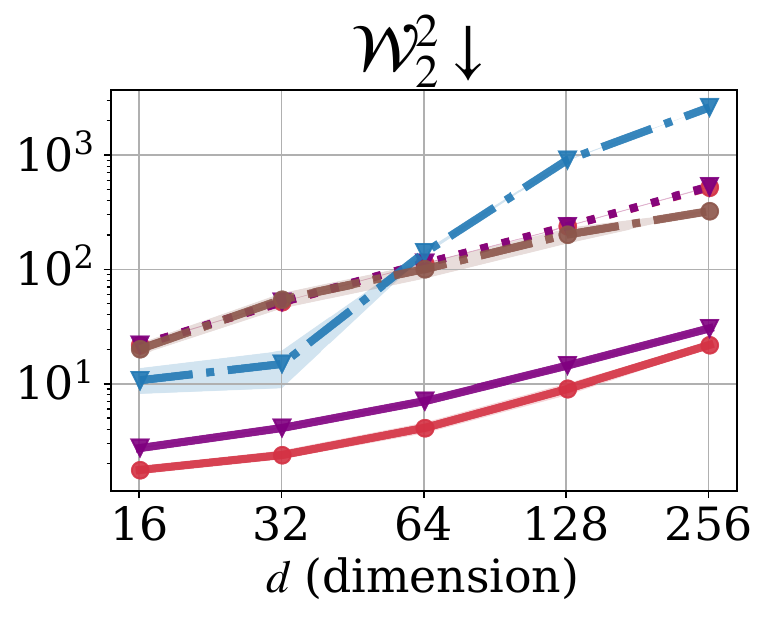}\hfill
        \includegraphics[width=.5\linewidth]{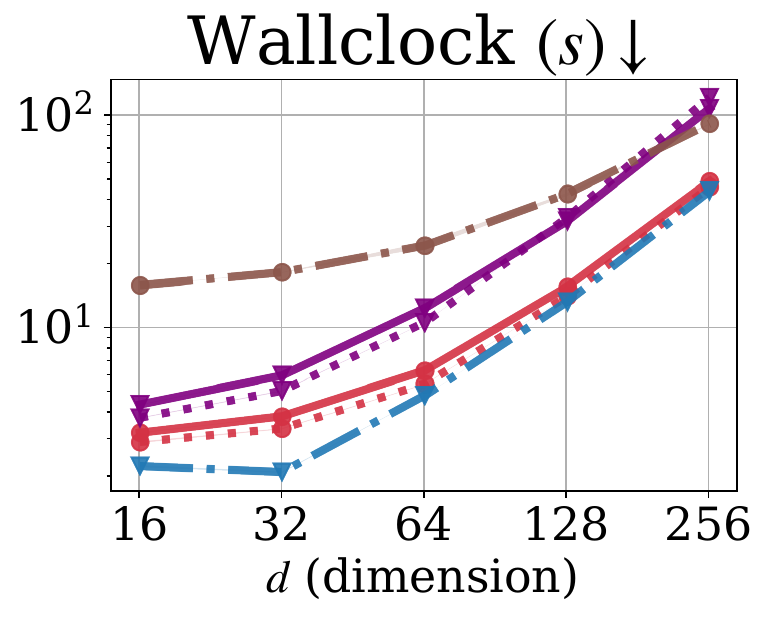}
        \caption{Increasing ambient and manifold dimension with fixed $m=5$.}
        \label{fig:scaling_inc_m}
  \end{subfigure}

    \caption{Sampling performance and wallclock time as the dimension $d$ increases (with $l = |C| = 5$). MASEM achieves lowest $\mathcal{W}_2^2$ (Sinkhorn) distances.}\vspace{5pt}
    \label{fig:scaling}
\end{figure}


\subsection{Robotics Applications}
\textbf{Motion Planning.} We consider trajectory sampling, where we sample constraint satisfying trajectories through an obstacle course. 
We sample the trajectories of a 2d pointmass across two different obstacle courses, one that is a regular $4\times 4$ grid and another one with 20 randomly placed obstacles (\Cref{fig:mp_lines}).
Each trajectory is parameterized by a spline defined by $3$ waypoints, resulting in a $6$ dimensional problem. 
Random obstacle motion planning has $934$ constraints in total, while grid based planning has $774$ constraints.

\textbf{Grasp Sampling.}  
Secondly, we evaluate MASEM for grasping. 
We optimize grasps on a capsule geometry (a common shape for collision checking) with three fingers. 
A band around the capsule shall not be touched, mirroring grasping in the real world where not all objects can be grasped at arbitrary points, e.g., a robot grasping a plate that must not put its finger into the food.
Each finger applies a force, such that in the end the fingers counteract gravity pulling on the object. This results in a $18$-dimensional problem, with $9$ equality and $48$ inequality constraints.

\begin{table}[h]
\centering
\small
\caption{Performance on robotics problems across seeds. Entries are mean $\pm$ 95\% confidence interval. Column-wise bold highlights indicate better than counterpart for t-test with p < 0.05 and Holm-Bonferroni correction.}
\begin{tabular}{lcccccc}
\toprule
Method & \multicolumn{2}{c}{Planning Random Obst.} & \multicolumn{2}{c}{Planning Grid Obst.} & \multicolumn{2}{c}{Grasping} \\
\cmidrule(lr){2-3} \cmidrule(lr){4-5} \cmidrule(lr){6-7}
 & Feas. Ent. $\uparrow$ & Slack $\downarrow$ & Feas. Ent. $\uparrow$ & Slack $\downarrow$ & Feas. Ent. $\uparrow$ & Slack $\downarrow$ \\
\midrule
NHR & $\tabpm{6.85}{\phantom{0}.37}$ & $\tabpm{\phantom{0}.00}{\phantom{0}.00}$ & $\tabpm{6.50}{\phantom{0}.14}$ & $\tabpm{\phantom{0}.00}{\phantom{0}.00}$ & $\tabpm{3.61}{\phantom{0}.23}$ & $\tabpm{\phantom{0}.00}{\phantom{0}.00}$ \\
Cluster-NHR & $\tabpm{6.49}{1.43}$ & $\tabpm{\phantom{0}.00}{\phantom{0}.00}$ & $\tabpm{4.41}{1.27}$ & $\tabpm{\phantom{0}.00}{\phantom{0}.00}$ & $\tabpm{3.60}{1.05}$ & $\tabpm{\phantom{0}.00}{\phantom{0}.00}$ \\
MASEM-NHR & $\mathbf{\tabpm{7.93}{\phantom{0}.22}}$ & $\tabpm{\phantom{0}.00}{\phantom{0}.00}$ & $\mathbf{\tabpm{7.49}{\phantom{0}.16}}$ & $\tabpm{\phantom{0}.00}{\phantom{0}.00}$ & $\mathbf{\tabpm{4.79}{\phantom{0}.16}}$ & $\tabpm{\phantom{0}.00}{\phantom{0}.00}$ \\
OLLA & $\tabpm{\phantom{0}.86}{\phantom{0}.15}$ & $\tabpm{\phantom{0}.07}{\phantom{0}.01}$ & $\tabpm{1.47}{\phantom{0}.24}$ & $\tabpm{\phantom{0}.03}{\phantom{0}.01}$ & $\tabpm{1.61}{\phantom{0}.09}$ & $\tabpm{\phantom{0}.37}{\phantom{0}.01}$ \\
MASEM-OLLA & $\tabpm{6.11}{\phantom{0}.47}$ & $\tabpm{\phantom{0}.01}{\phantom{0}.00}$ & $\tabpm{5.35}{\phantom{0}.38}$ & $\tabpm{\phantom{0}.01}{\phantom{0}.00}$ & $\tabpm{4.09}{\phantom{0}.18}$ & $\tabpm{\phantom{0}.01}{\phantom{0}.00}$ \\
SCMC & $\tabpm{1.43}{\phantom{0}.69}$ & $\tabpm{\phantom{0}.05}{\phantom{0}.04}$ & $\tabpm{2.26}{\phantom{0}.26}$ & $\tabpm{\phantom{0}.02}{\phantom{0}.01}$ & $\tabpm{\phantom{0}.22}{\phantom{0}.04}$ & $\tabpm{\phantom{0}.08}{\phantom{0}.00}$ \\
\bottomrule
\end{tabular}
\label{tab:robotics}
\end{table}
\textbf{Results.} Due to the lack of ground truth samples, we report the feasible entropy instead of the Sinkhorn distance.
We compute it by multiplying the entropy of the feasible samples and the fraction of feasible samples (made precise in \Cref{app:metrics}).
\Cref{tab:robotics} shows an increase of sample entropy when applying MASEM, regardless of sampler.
This is illustrated in \Cref{fig:mp_lines} for random obstacles. We observe that MASEM-NHR samples trajectories through the narrow gaps at the center of the course, which are missed by standard NHR.
For OLLA the performance gap is even larger, as the default sampler only generates a locally distributed sample sets but fails to sample the outer regions of the domain. Moreover, OLLA generates many paths crossing over obstacles. 
%
%
While MASEM-NHR outperforms MASEM-OLLA on all problems, OLLA benefits more from the resampling applied by MASEM. 
This suggests two things: first, NHR-based samplers are more suited for robotics than OLLA-based samplers (at least within this evaluation set), and secondly, MASEM greatly improves previously unsuited samplers and makes them competitive in this domain. We believe these are valuable for future research on sampling in robotics. 
%
Further, our approach outperforms alternative global sample allocation strategies: SCMC performs poorly on all tasks;  MASEM-NHR outperforms its clustering counterpart on all tasks as well, with a large gap on the regular grid design.
The performance of MASEM on the motion planning problems demonstrates its ability to scale to problems with a high number of constraints.

\begin{figure}[t]
    \centering
    \includegraphics[width=\linewidth]{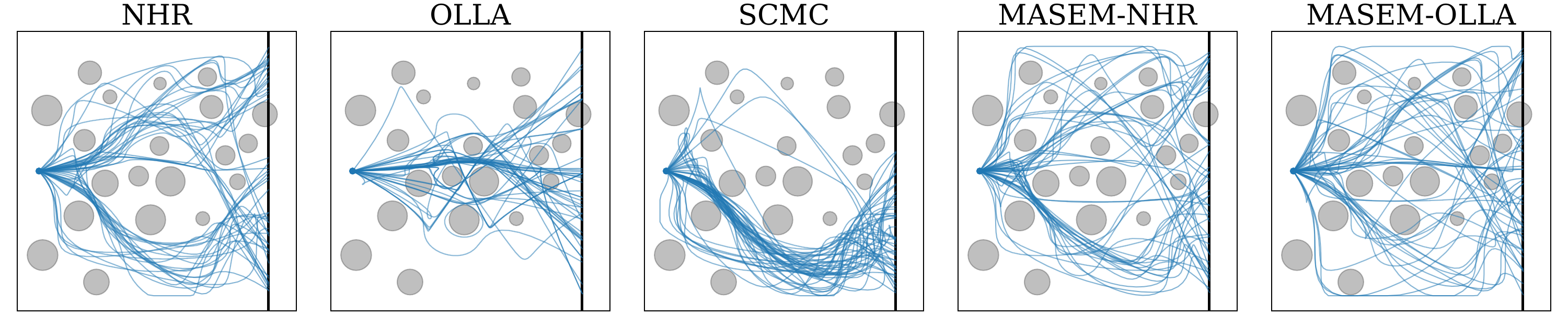}
    \caption{Plots of 50 samples on the random motion planning problem. MASEM-based approaches yield the highest sample entropy across all methods.}
    \label{fig:mp_lines}
\end{figure}

\section{Conclusion \& Future Work}
\paragraph{Summary}
We presented Manifold Sampling via Entropy Maximization (MASEM), a resampling-based approach that can be combined with any constrained MCMC sampler to enable uniform sampling on disconnected manifolds implicitly defined by constraints.
We proved that MASEM minimizes the KL divergence to the uniform target distribution exponentially in the number of resampling steps and provided worst-case bounds in the mean-field.
Building on this, we analyzed two instantiations of MASEM on several synthetic and real world problems, where it results in a significantly lower Sinkhorn distance to the ground truth distribution compared to the baselines. 

\paragraph{Limitations} 
While we believe that MASEM is a useful constrained sampling framework, it is not without limitations. 
Some results, namely \Cref{lem:inital-populated}, \Cref{thm:main}, and \Cref{cor:convergence}, are valid only in the large-scale particle limit, i.e.~for $N\to\infty$. 
Further analysis is warranted to describe the algorithms convergence under finite $N$.
Beyond theoretical limitations, there are two practical considerations when using MASEM in combination with an existing sampler. 
First, the resampling step introduces a small runtime overhead, as one needs to compute the pairwise distances to compute the resampling weights. 
In practice however, we discovered that on our benchmarks this overhead is small, while applying MASEM greatly enhances sample quality. 
We refer to \Cref{app:limitations} for a more in-depth discussion on the limitations of our method. 

\paragraph{Future Work}
In this work we focus on entropy maximization, which leads to uniform sampling across the manifold. 
Future work could explore generalizations to non-uniform distributions. We see potential in applying MASEM as a data generation method in robotics and other application domains.
An intriguing avenue of research is investigating practical design choices and comparing them to current heuristic approaches in practical applications.

\FloatBarrier
\begin{ack}
The authors sincerely thank Paula Cordero Encinar for discussions during early stages of the project, and Eckart Cobo-Briesewitz for his feedback on the manuscript.
This research was funded by the Amazon Fulfillment Technologies and Robotics team.
This work has been supported by the German Federal Ministry of Research, Technology and Space (BMFTR) under the Robotics Institute Germany (RIG).
\end{ack}

\newpage
\bibliography{references-uppercase-titles}
\bibliographystyle{unsrtnat}

\newpage
\appendix\crefalias{section}{appendix}
\crefalias{subsection}{appendix}
\section{Formal Lemmas and Proofs}\label{app:proofs}
We begin by formalizing \Cref{lemma:ent_kl} introduced in \Cref{sec:prelim}.
\setcounter{lemma}{0}
\begin{lemma}
    Let $\Sigma$ be measurable with $0<\sigma_\Sigma(\Sigma)<\infty$, and let
    \begin{equation*}
    u_\Sigma(x)\coloneq \frac{1}{S}\quad\text{for }S=\sigma_\Sigma(\Sigma)
    \end{equation*}
    denote the uniform density on $\Sigma$ with respect to $\sigma_\Sigma$.
    Then it holds that $\log S = H(u_\Sigma)\geq \Hent(\rho_\Sigma)$ for all densities $\rho_\Sigma$ supported on $\Sigma$ with equality iff $\rho = u_\Sigma$ almost everywhere.
\end{lemma}
\begin{proof}
    We compute the entropy of the uniform density as 
    \begin{equation*}
        \Hent(u_\Sigma)
        = - \int_\Sigma \frac{1}{S} 
           \log \left(\frac{1}{S}\right) d\sigma_\Sigma
        = \log S
    \end{equation*}
    Next, consider the relative entropy of $\rho$ \wrt $u_\Sigma$:
    \begin{align*}
        D_{KL}(\rho\| u_\Sigma)
        &=\int_\Sigma \rho\log\frac{\rho}{u_\Sigma}d\sigma_\Sigma
        =\int_\Sigma \rho\log S~d\sigma_\Sigma + \int_\Sigma \rho \log \rho~d\sigma_\Sigma \\
        &= \log S - \Hent(\rho)  = \Hent(u_\Sigma)-\Hent(\rho).
    \end{align*}
    Now, by Gibbs' inequality $0\leq D_{KL}(\rho \| u_\Sigma)$
    with equality iff $\rho=u_\Sigma$ almost everywhere, proving the statement.
\end{proof}
Next, we restate the results of \Cref{sec:analysis} and provide the complete proofs.

\printProofs

\section{Additional Figures}\label{app:plots}
\subsection{Synthetic Benchmarks Convergence and Qualitative Results}
We list the full learning curves in \Cref{fig:syn_conv}. 
These curves illustrate the evolution of Sinkhorn and KL divergence as well as constraint violation. More details on the metrics can be found in \Cref{app:metrics}.
Overall we observe strong performance and fast convergence of both MASEM variants on all problems.
While KL divergence is low for all methods on the connected disk problem, separating the disks as depicted in \Cref{fig:2dbenchs} leads to a performance deterioration of all baselines as shown in \Cref{fig:syn_conv}. Even for the connected disks, MASEM leads to faster convergence.

\newpage
\begin{figure}[p]
    \centering
    \includegraphics[width=.8\linewidth]{content/figures/standalone_legend.pdf}
\begin{minipage}{\textwidth}
    \begin{minipage}{0.05\linewidth}
        \centering
        \rotatebox{90}{\small Connected Disks}
    \end{minipage}
    \begin{minipage}{0.9\linewidth}
        \begin{subfigure}{\linewidth}
            \includegraphics[width=\linewidth, trim={.2cm .2cm 0.2cm 0.2cm}, clip]{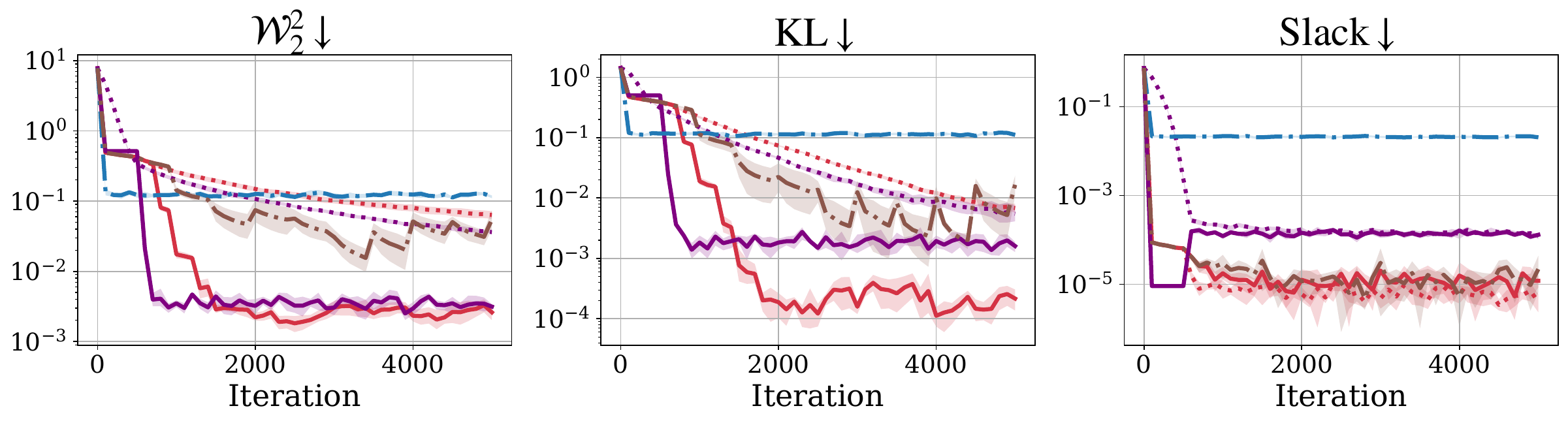}
        \end{subfigure}
    \end{minipage}
\end{minipage}

\begin{minipage}{\textwidth}
    \begin{minipage}{0.05\linewidth}
        \centering
        \rotatebox{90}{\small Disconnected Disks}
    \end{minipage}
    \begin{minipage}{0.9\linewidth}
        \begin{subfigure}{\linewidth}
            \includegraphics[width=\linewidth, trim={.2cm .2cm 0.2cm 0.2cm}, clip]{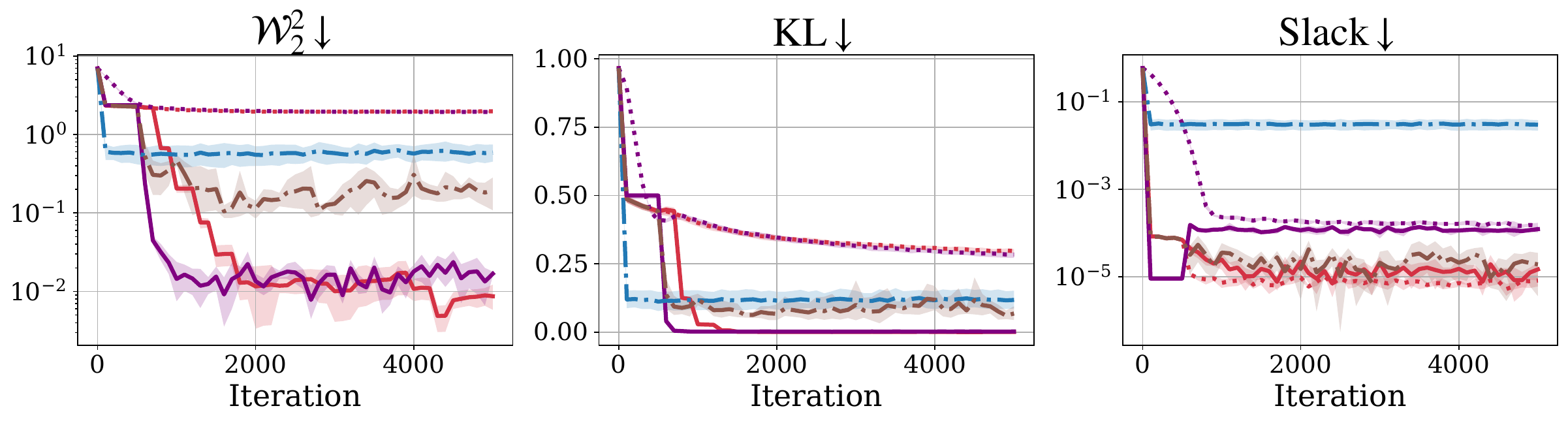}
        \end{subfigure}
    \end{minipage}
\end{minipage}

\begin{minipage}{\textwidth}
    \begin{minipage}{0.05\linewidth}
        \centering
        \rotatebox{90}{\small Seven lobes}
    \end{minipage}
    \begin{minipage}{0.9\linewidth}
        \begin{subfigure}{\linewidth}
            \includegraphics[width=\linewidth, trim={.2cm .2cm 0.2cm 0.2cm}, clip]{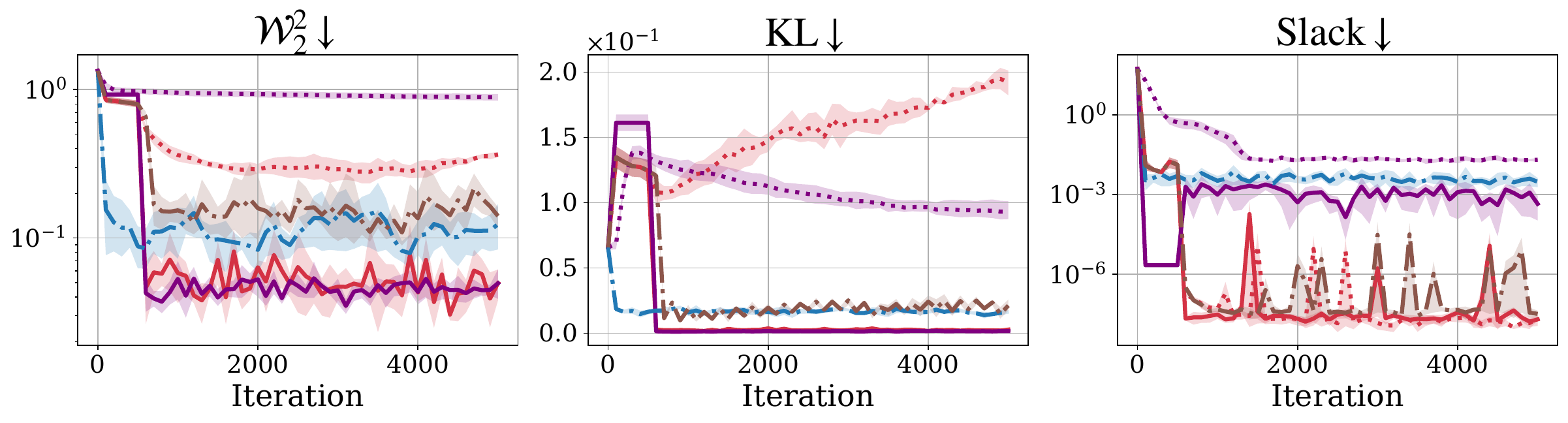}
        \end{subfigure}
    \end{minipage}
\end{minipage}

\begin{minipage}{\textwidth}
    \begin{minipage}{0.05\linewidth}
        \centering
        \rotatebox{90}{\small Sine}
    \end{minipage}
    \begin{minipage}{0.9\linewidth}
        \begin{subfigure}{\linewidth}
            \includegraphics[width=\linewidth, trim={.2cm .2cm 0.2cm 0.2cm}, clip]{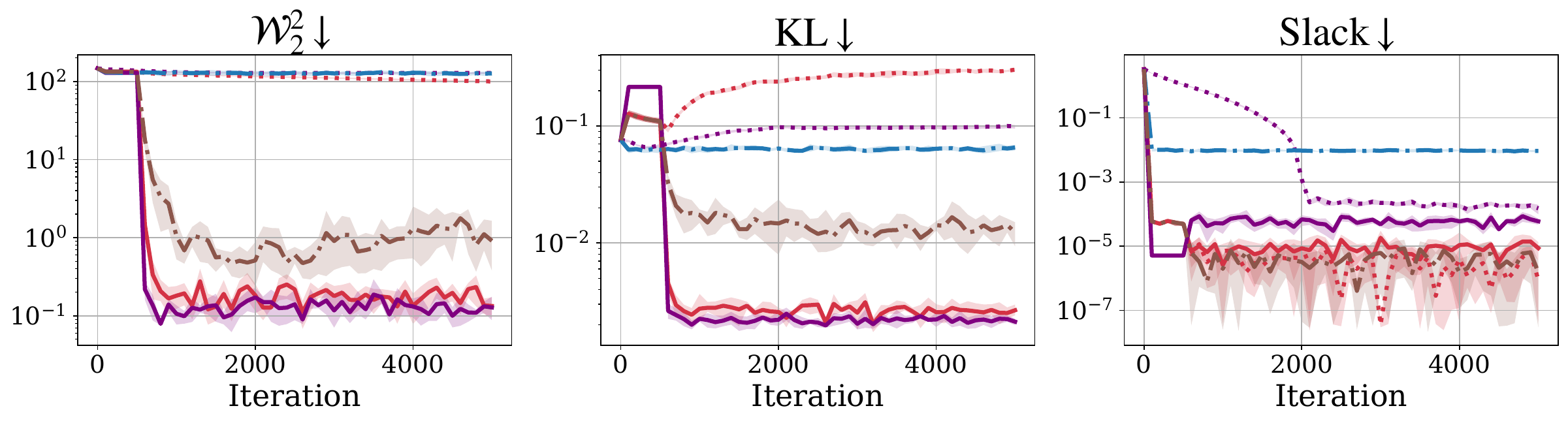}
        \end{subfigure}
    \end{minipage}
\end{minipage}

\begin{minipage}{\textwidth}
    \begin{minipage}{0.05\linewidth}
        \centering
        \rotatebox{90}{\small Swiss roll}
    \end{minipage}
    \begin{minipage}{0.9\linewidth}
        \begin{subfigure}{\linewidth}
            \includegraphics[width=\linewidth, trim={.2cm .2cm 0.2cm 0.2cm}, clip]{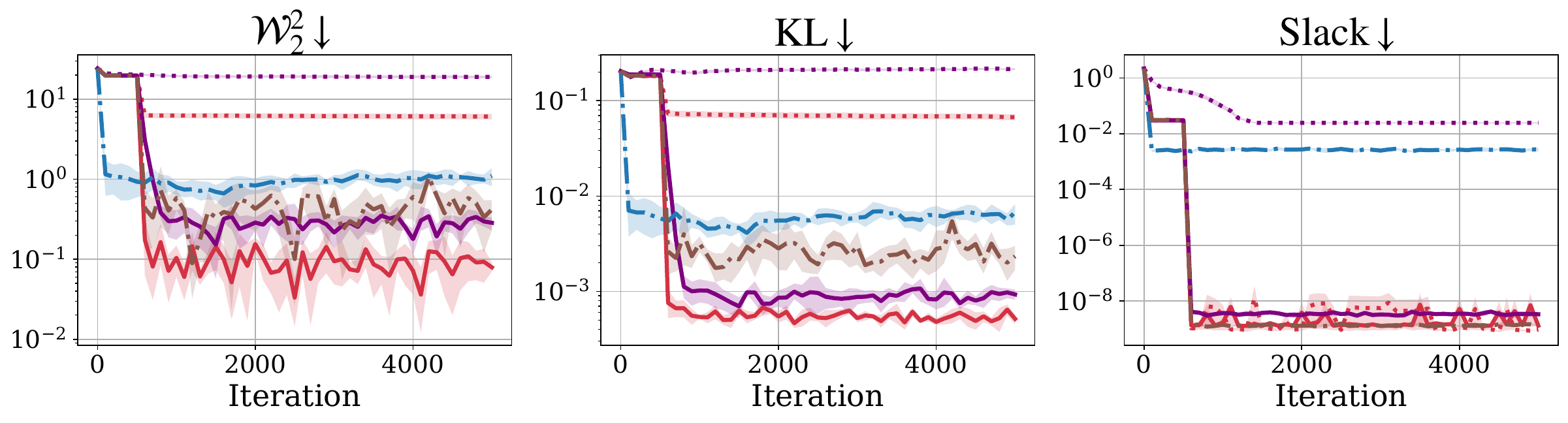}
        \end{subfigure}
    \end{minipage}
\end{minipage}
\caption{Convergence over 5000 iterations on synthetic benchmarks. From left to right: (1) squared Sinkhorn distance to ground truth samples, (2) KL divergence to pairwise distances of ground truth samples, and (3) mean maximum constraint violation $\mathbb{E}_x[\max \{|h(x)|, g^+(x)\}]$. Solid lines and shaded bands show the mean and $95\%$ CI over five independent runs. Both MASEM-NHR and MASEM-OLLA quickly decrease constraint violation and maintain it there, while achieving the lowest $\mathcal{W}_2^2$ and $\text{KL}$ values.}
    \label{fig:syn_conv}
\end{figure}

\begin{figure}[t]
    \centering
    \includegraphics[width=\linewidth, trim={.2cm .2cm 0.2cm 0.2cm}, clip]{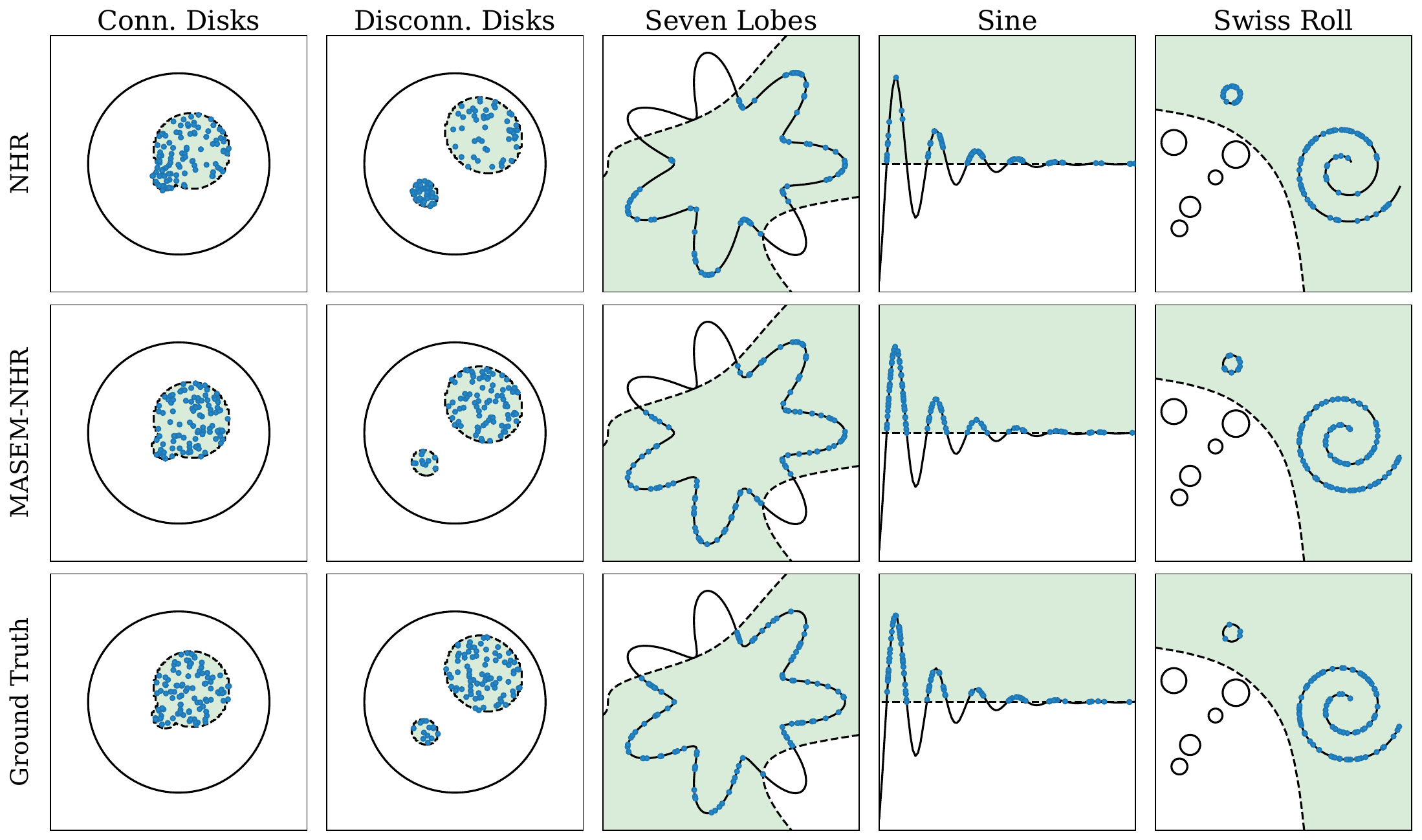}
    \caption{Qualitative results for NHR and MASEM-NHR on all 2d benchmark problems. Black lines show equality constraints, and green shaded areas mark the
feasible region(s) by inequality constraints. NHR struggles with moving across components, as can be seen in the disconnected disk and the sine benchmark. Moreover, on manifolds with high curvature, NHR also struggles to distribute particles correctly, since it works by locally linearizing the equality constraints. This can be seen with the seven lobes and the swiss roll benchmark. In both cases, MASEM is able to improve the distribution of particles.}
    \label{fig:2dbenchs}
\end{figure}

\newpage
\subsection{Scaling Analysis}
\begin{figure}[H]
\centering
\includegraphics[width=0.8\textwidth]{content/figures/standalone_legend.pdf}
\begin{minipage}{\textwidth}
    \begin{minipage}{0.05\linewidth}
        \centering
        \rotatebox{90}{\parbox{2cm}{\centering\small Scaling $d$ ($m=5$)}}
    \end{minipage}
    \begin{minipage}{0.94\linewidth}
        \begin{subfigure}{\linewidth}
        \centering
\centering
    \includegraphics[width=0.25\linewidth]{content/figures/scaling/disconnected_stress_sd_proj.pdf}\hfill
    \includegraphics[width=0.25\linewidth]{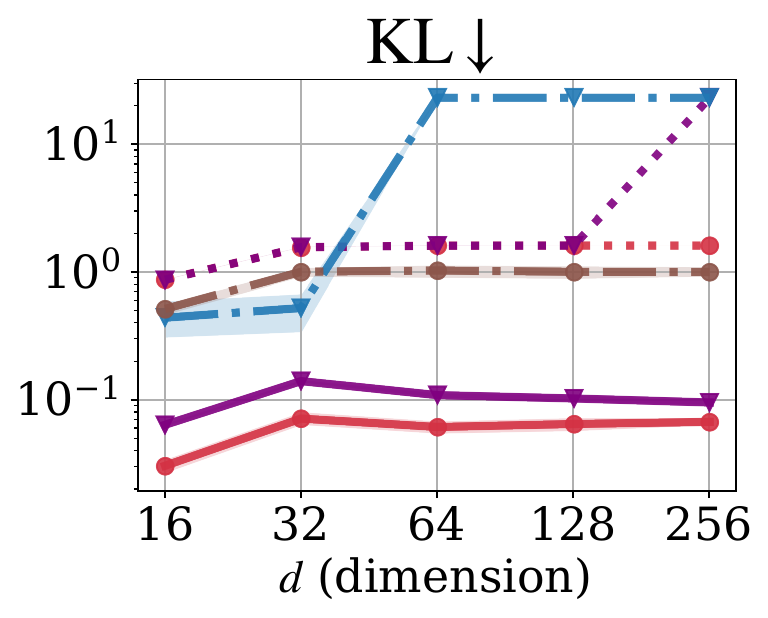}\hfill
    \includegraphics[width=0.25\linewidth]{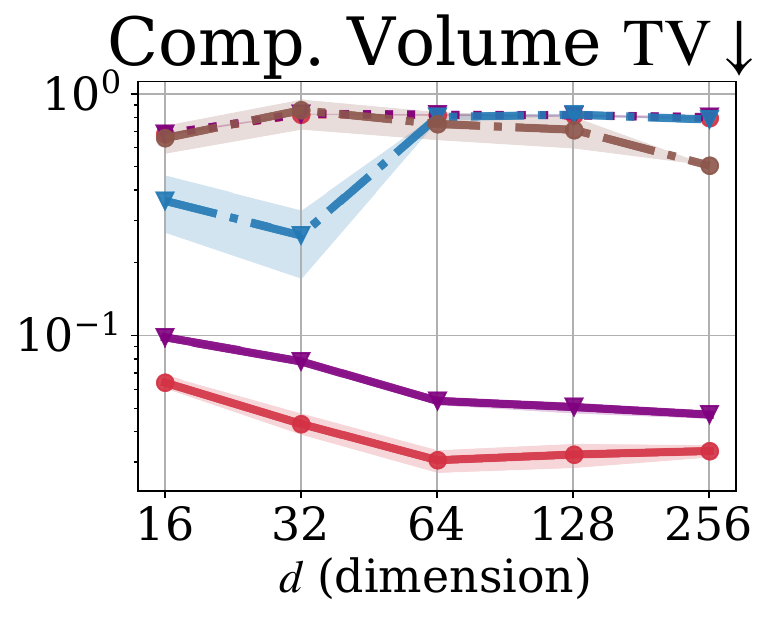}\hfill
    \includegraphics[width=0.25\linewidth]{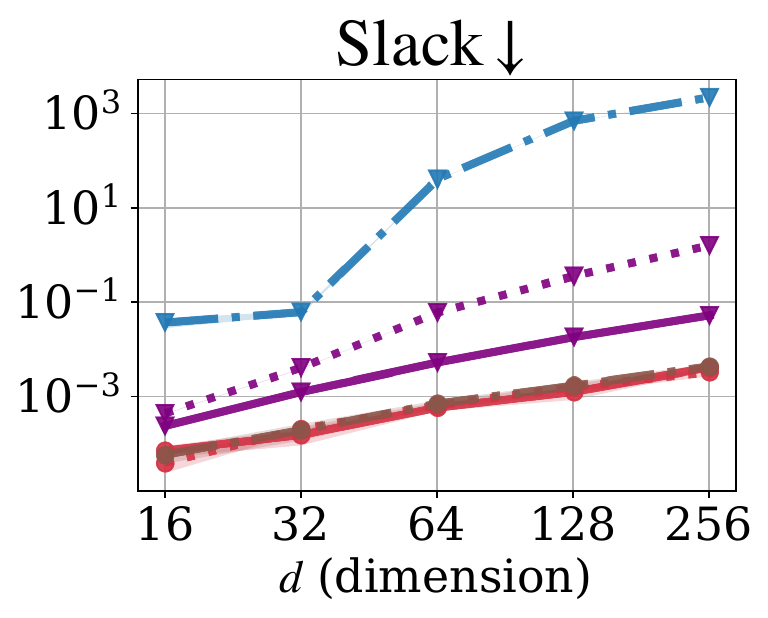}
        \end{subfigure}
    \end{minipage}
\end{minipage}
\begin{minipage}{\textwidth}
    \begin{minipage}{0.05\linewidth}
        \centering
        \rotatebox{90}{\parbox{2cm}{\centering\small Scaling $d$ and $m=d-3$}}
    \end{minipage}
    \begin{minipage}{0.94\linewidth}
        \begin{subfigure}{\linewidth}
        \centering
    \includegraphics[width=0.25\linewidth]{content/figures/scaling/disconnected_stress_2d_sd_proj.pdf}\hfill
    \includegraphics[width=0.25\linewidth]{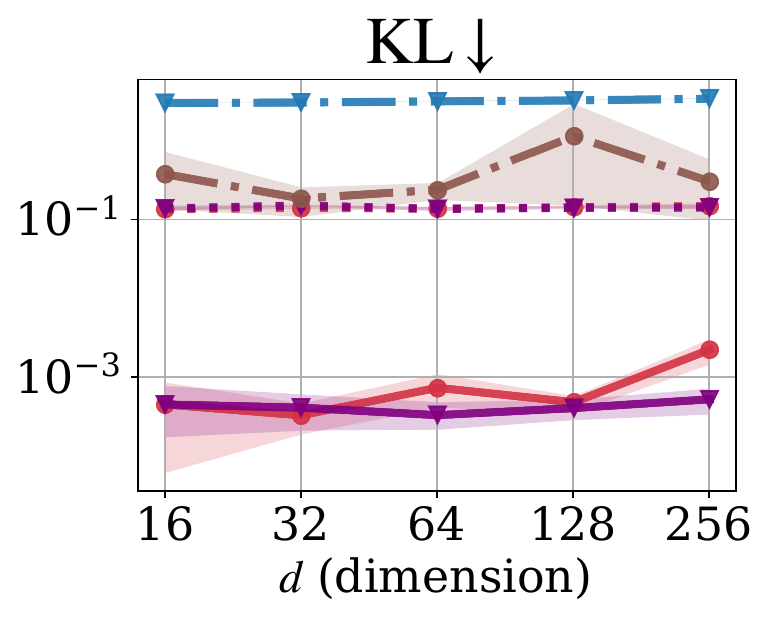}\hfill
    \includegraphics[width=0.25\linewidth]{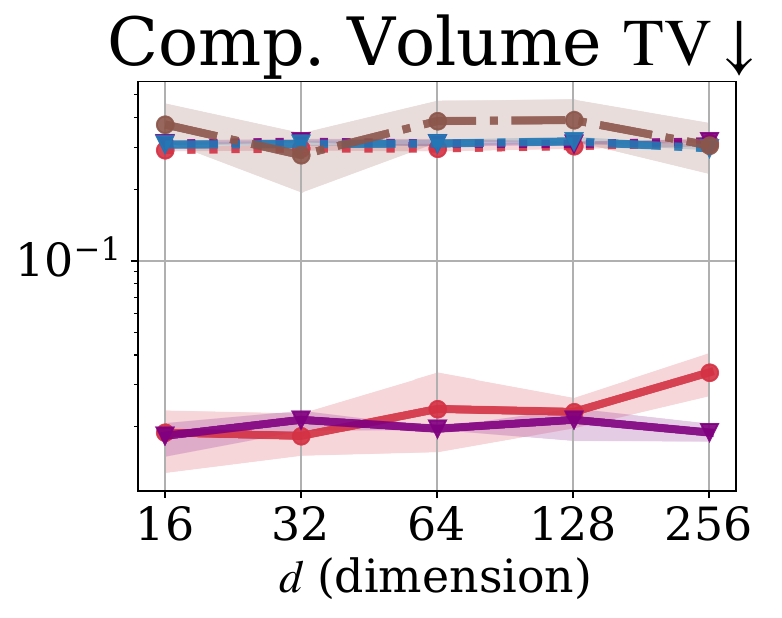}\hfill
    \includegraphics[width=0.25\linewidth]{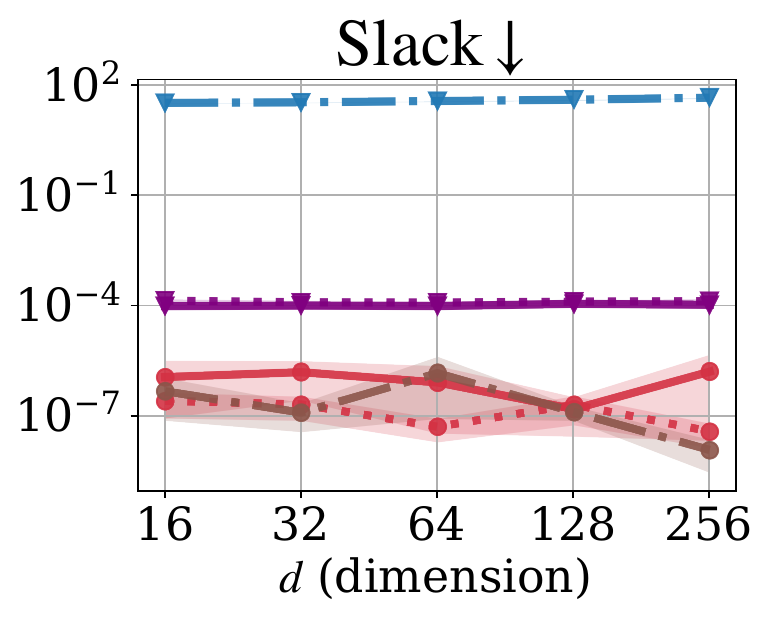}
        \end{subfigure}
    \end{minipage}
\end{minipage}
    
    \caption{Stress test for scaling ambient and manifold dimension (first row) and scaling the ambient dimension while the manifold dimension stays fixed. (second row). Curves depict mean and 95\% CI across 5 seeds. We see that MASEM outperforms alternative approaches by an order of magnitude. Performances are stable across ambient dimensions, indicating that manifold dimension is the main difficulty for samplers. Soft constraint approaches suffer in particular with high-dimensional manifolds.}
    \label{fig:scaling_conv}
\end{figure}

We report the full scaling results in \Cref{fig:scaling_conv}.
The plots report Sinkhorn distance and pairwise distance KL divergence \wrt to ground truth samples.
In addition, we estimate the total variation distance (TV) to the ground truth component volume estimates of the resulting empirical distributions as well as expected maximum constraint violations.
We observe that soft constraint approaches scale poorly on problems with high ambient dimension, likely due to the high number of constraints that must be annealed.
Across dimensions MASEM outperforms its NHR and OLLA counterparts for both problems.

\subsection{Robotics Problems Convergence Curves}
We provide the convergence curves for the robotics problems. In \Cref{fig:mp_conv} we report the convergence for motion planning. In addition to the feasible entropy and the slack, we report the entropy of path homotopy classes. This corresponds approximately to the entropy of the component weights (for more details see \Cref{par:homo-entro}). In \Cref{fig:grasping_conv} we report the convergence curves for the grasping problem.
\begin{figure}[H]
    \centering
    \includegraphics[width=.8\linewidth]{content/figures/standalone_legend.pdf}
\begin{minipage}{\textwidth}
    \begin{minipage}{0.05\linewidth}
        \centering
        \rotatebox{90}{\small Grid Obstacles}
    \end{minipage}
    \begin{minipage}{0.9\linewidth}
        \begin{subfigure}{\linewidth}
            \includegraphics[width=\linewidth]{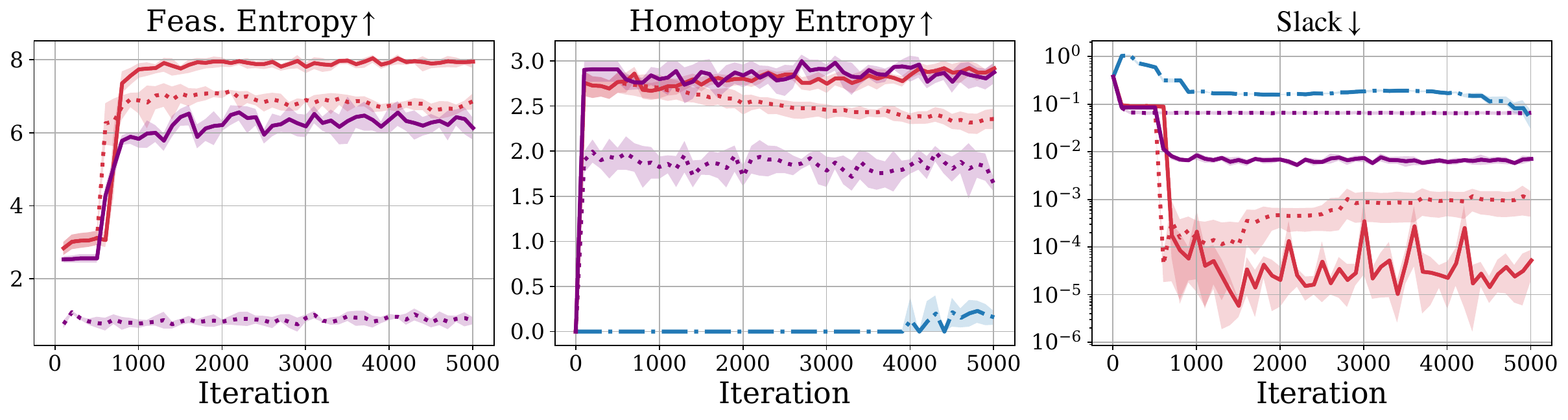}
        \end{subfigure}
    \end{minipage}
\end{minipage}

\begin{minipage}{\textwidth}
    \begin{minipage}{0.05\linewidth}
        \centering
        \rotatebox{90}{\small Random Obstacles}
    \end{minipage}
    \begin{minipage}{0.9\linewidth}
        \begin{subfigure}{\linewidth}
            \includegraphics[width=\linewidth]{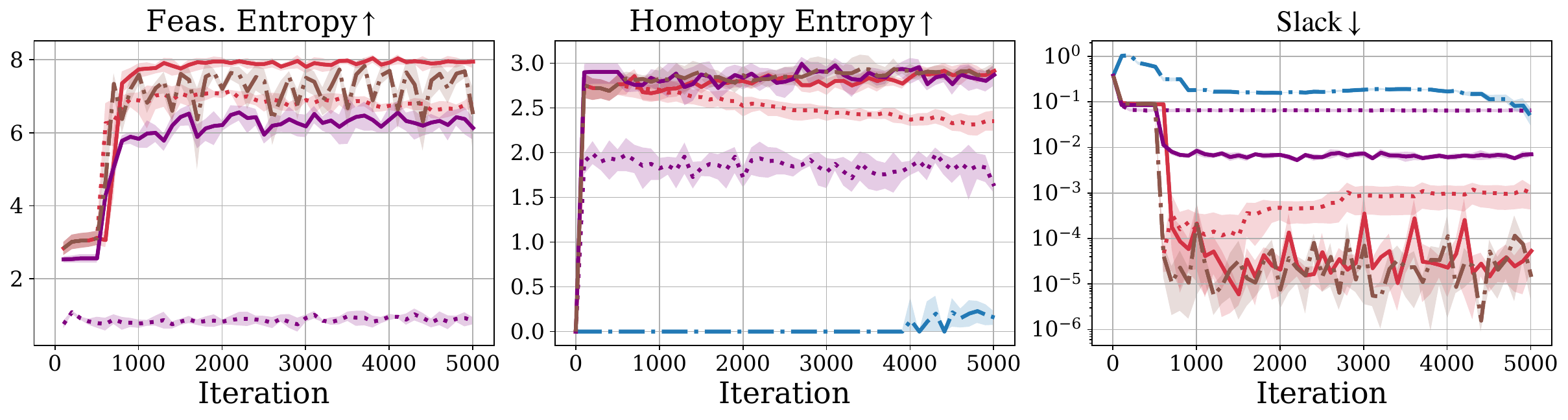}
        \end{subfigure}
    \end{minipage}
\end{minipage}
    \caption{Convergence curves on the motion planning problems with grid obstacles (first row) and random obstacles (second row). Curves depict mean and $95\%$ CI
across 5 seeds.}
    \label{fig:mp_conv}
\end{figure}

\begin{figure}[H]
    \centering
    \includegraphics[width=.8\linewidth]{content/figures/standalone_legend.pdf}
    \includegraphics[width=0.85\linewidth]{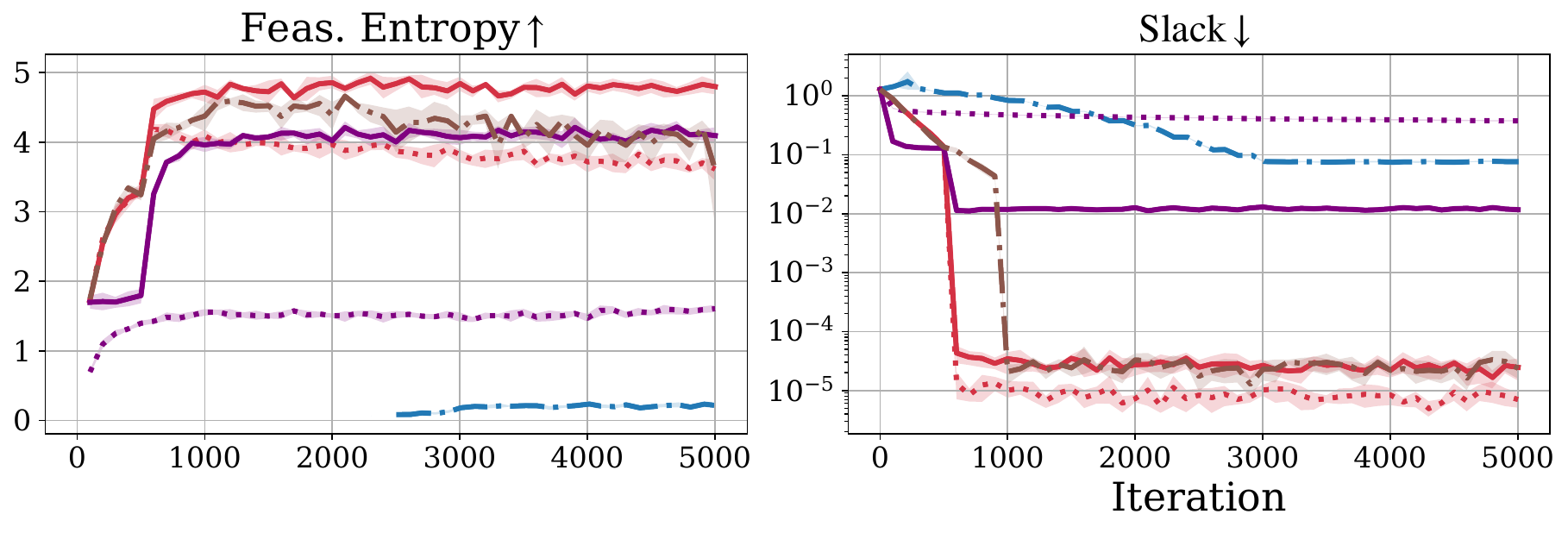}
    \caption{Convergence on the grasping problem. Curves depict mean and $95\%$ CI
across 5 seeds.}
    \label{fig:grasping_conv}
\end{figure}

\newpage
\section{Additional Results}
\subsection{Analysis of Recommended Parameters}\label{app:ana-paras}
We recommend the standard parameters\begin{equation}
    \tau=0.3,\ M=20,\ \text{and }k=20 \tag{\ref{eq:standardparams}}
\end{equation}
as an initial guess for using our framework. To demonstrate their use, we compare the Sinkhorn divergence on all benchmarks that have a ground truth against the tuned sampler. We note that this analysis was performed post-hoc: at the time of choosing these parameters, the benchmarks and the tuned parameters were already known. Nevertheless, we believe this is a useful validation of our recommended parameters. 
\begin{wraptable}{R}{0.4\textwidth}
\centering
\small
\caption[Additional results for curvature scaling]{Additional results for the curvature scaling experiment at $\lvert C\rvert=5$. For polynomials $p$ and $q$ of degree $1,\dots,6$, we provide the mean of the mean and max curvature, and the baseline Sinkhorn divergence measured by (tuned) NHR and averaged over five runs.}
    \begin{tabular}{cccc}
    \toprule
        \multirow{2}{*}{Deg.} & mean & max & \multirow{2}{*}{NHR $\mathcal{W}_2^2\downarrow$} \\
         & curv. & curv. & \\
        \midrule
        1 & $1.42$  &  $\hphantom{0}1.42$ & $\tabpm{2.57}{\hphantom{0}.24}$ \\
        2 & $3.48$ & $\hphantom{0}4.14$ & $\tabpm{1.54}{\hphantom{0}.24}$ \\
        3 & $5.73$ & $\hphantom{0}8.08$ &  $\tabpm{3.81}{\hphantom{0}.40}$ \\
        4 & $5.86$ & $10.01$ & $\tabpm{4.14}{1.05}$ \\
        5 & $5.95$ & $11.77$ & $\tabpm{3.77}{\hphantom{0}.42}$ \\
        6 & $8.00$ & $19.12$ & $\tabpm{3.98}{\hphantom{0}.39}$ \\
        \bottomrule
    \end{tabular}
    \label{tab:curv-data}
\end{wraptable}

\Cref{tab:tuned-v-fixed} shows the results of this analysis. We see that for all benchmarks, the fixed parameters land within one order of magnitude of the tuned sampler.
\begin{table}[t]
    \centering
    \caption{Comparison of sampling using untuned resampling with the default parameters \eqref{eq:standardparams} against the tuned sampler.}
\begin{tabular}{lcccc}
    \toprule
    Problem & \multicolumn{2}{c}{MASEM-NHR} & \multicolumn{2}{c}{MASEM-OLLA} \\
    \cmidrule(lr){2-3}\cmidrule(lr){4-5}
    & Fixed & Tuned & Fixed & Tuned \\
    \midrule
    Connected Disks &
    $\tabpm{0.00}{0.00}$ &
    $\tabpm{0.00}{0.00}$ &
    $\tabpm{0.01}{0.00}$ &
    $\tabpm{0.00}{0.00}$ \\
    Disconnected Disks &
    $\tabpm{0.01}{0.01}$ &
    $\tabpm{0.01}{0.01}$ &
    $\tabpm{0.03}{0.01}$ &
    $\tabpm{0.02}{0.00}$ \\
    Seven Lobes &
    $\tabpm{0.03}{0.01}$ &
    $\tabpm{0.05}{0.02}$ &
    $\tabpm{0.04}{0.01}$ &
    $\tabpm{0.05}{0.02}$ \\
    Sine &
    $\tabpm{0.50}{0.07}$ &
    $\tabpm{0.13}{0.06}$ &
    $\tabpm{0.13}{0.06}$ &
    $\tabpm{0.13}{0.07}$ \\
    Swiss Roll &
    $\tabpm{0.25}{0.12}$ &
    $\tabpm{0.08}{0.02}$ &
    $\tabpm{0.23}{0.08}$ &
    $\tabpm{0.29}{0.09}$ \\
    \midrule
    Stress Test ($m=5$, $d=256$) &
    $\tabpm{25.21}{1.38}$ &
    $\tabpm{21.86}{23.40}$ &
    $\tabpm{28.10}{2.60}$ &
    $\tabpm{30.50}{31.52}$ \\
    Stress Test ($m=d-3$, $d=256$) &
    $\tabpm{0.86}{0.08}$ &
    $\tabpm{0.14}{0.18}$ &
    $\tabpm{0.10}{0.02}$ &
    $\tabpm{0.09}{0.10}$ \\
    \bottomrule
\end{tabular}
    \label{tab:tuned-v-fixed}
\end{table}

\subsection{Performance on High Curvature Manifolds}
To test the performance on high-curvature manifolds, we evaluate MASEM on a variant of the high-dimensional stress test. 
We embed two-dimensional disks with radius $r$ in a $d=16$-dimensional ambient euclidean space. For each disk, we sample two distorting polynomials $p_i, q_i$ for $i=1,\dots,d-2$ for the two axes of the disk with coefficients uniform in $[-1,1]$, and embed the graph of the disk under the product of these distortions $f\colon \mathbb{R}^2\to\mathbb{R}^{d-2}, f(x,y)_i = p_i(x)q_i(y)$. 
Last, we rotate the distorted disk by multiplying its points with a randomly generated orthonormal matrix. To measure the effect of the polynomial degree on the curvature, we use the Frobenius norm of the Hessian $\lVert H_f(x,y)\rVert_F$ as a measure of local curvature and calculate the mean and max curvature of the disk by Monte Carlo Estimation (see \Cref{tab:curv-data}). 
\Cref{tab:comps-v-curvature} shows the final Sinkhorn divergence when using a growing number of components and different polynomial degrees. We use MASEM-NHR with the standard parameters \eqref{eq:standardparams}. Additionally, we provide the baseline performance of tuned NHR for $5$ components in \Cref{tab:curv-data}. We see that on five components, MASEM-NHR performs worse at higher curvature, but still beats the baseline. Performance degrades gently as the number of components and the curvature grows.

\begin{table}[tbp]
    \centering
        \caption[Sinkhorn divergence for increasing curvature and number of components]{Final $\mathcal{W}_2^2$ when using MASEM-NHR standard parameters \eqref{eq:standardparams} on varying number of components and curvature. Values are mean over $5$ seeds $\pm$ 95\% CI.}
    \label{tab:comps-v-curvature}
    \begin{tabular}{lcccccc}
    \toprule
   $\mathcal{W}_2^2\downarrow$ & \multicolumn{6}{c}{polynomial degree of $p$ and $q$}\\
          & 1 & 2 & 3 & 4 & 5 & 6\\
         \midrule       
$\lvert C\rvert=\hphantom{0}5$ & $\tabpm{\hphantom{0}.12}{.04}$ & $\tabpm{\hphantom{0}.17}{.01}$ & $\tabpm{\hphantom{0}.69}{.10}$ & $\tabpm{\hphantom{0}.60}{.06}$ & $\tabpm{1.41}{.11}$ & $\tabpm{1.58}{.08}$\\
$\lvert C\rvert=10$ & $\tabpm{\hphantom{0}.16}{.04}$ & $\tabpm{\hphantom{0}.25}{.03}$ & $\tabpm{\hphantom{0}.64}{.04}$ & $\tabpm{\hphantom{0}.94}{.06}$ & $\tabpm{1.57}{.12}$ & $\tabpm{1.34}{.04}$\\
$\lvert C\rvert=15$ & $\tabpm{\hphantom{0}.13}{.03}$ & $\tabpm{\hphantom{0}.77}{.07}$ & $\tabpm{\hphantom{0}.82}{.05}$ & $\tabpm{1.82}{.07}$ & $\tabpm{2.76}{.07}$ & $\tabpm{6.71}{.04}$\\
$\lvert C\rvert=20$ & $\tabpm{\hphantom{0}.26}{.04}$ & $\tabpm{1.05}{.08}$ & $\tabpm{1.71}{.14}$ & $\tabpm{2.71}{.08}$ & $\tabpm{3.56}{.08}$ & $\tabpm{6.75}{.08}$\\
\bottomrule
\end{tabular}
\end{table}

\subsection{Comparison of Theoretical and Empirical Convergence Rate}
We plot the empirical reduction of component weight error for MASEM and compare it to the mean field bound $C_0(1-\tau/p)^{2t}$ from \cref{thm:main} in \Cref{fig:mf_rate}.
We observe across all dimensions that MASEM rapidly reduces the initial component-weight error geometrically. 
In $d=64$ and $d=256$, the empirical KL remains below the mean-field upper bound throughout the evaluated horizon. 
In $d=16$, the empirical curve eventually crosses the idealized bound after approximately $1500-2000$ iterations and stabilizes around for MASEM-NHR and for OLLA. 
These results show that the mean-field contraction accurately describes the initial correction, after which finite-particle and approximate-mixing effects determine the residual accuracy. 
Importantly, the curves remain stable and most of the initial component imbalance has already been removed before this regime is reached. 
In addition, we see that the feasibility correction has little effect in lower dimensions. 
In the more challenging setting, it substantially improves the final KL for both MASEM variants, which supports its role as a practical robustness mechanism when approximate feasibility becomes more consequential. 
We also observe that the NHR base sampler does not move samples into infeasible regions by design. Therefore, MASEM-NHR does not benefit from the infeasibility-reweighting as much as MASEM-OLLA.

\begin{figure}[b]
    \centering
        \includegraphics[width=.95\linewidth]{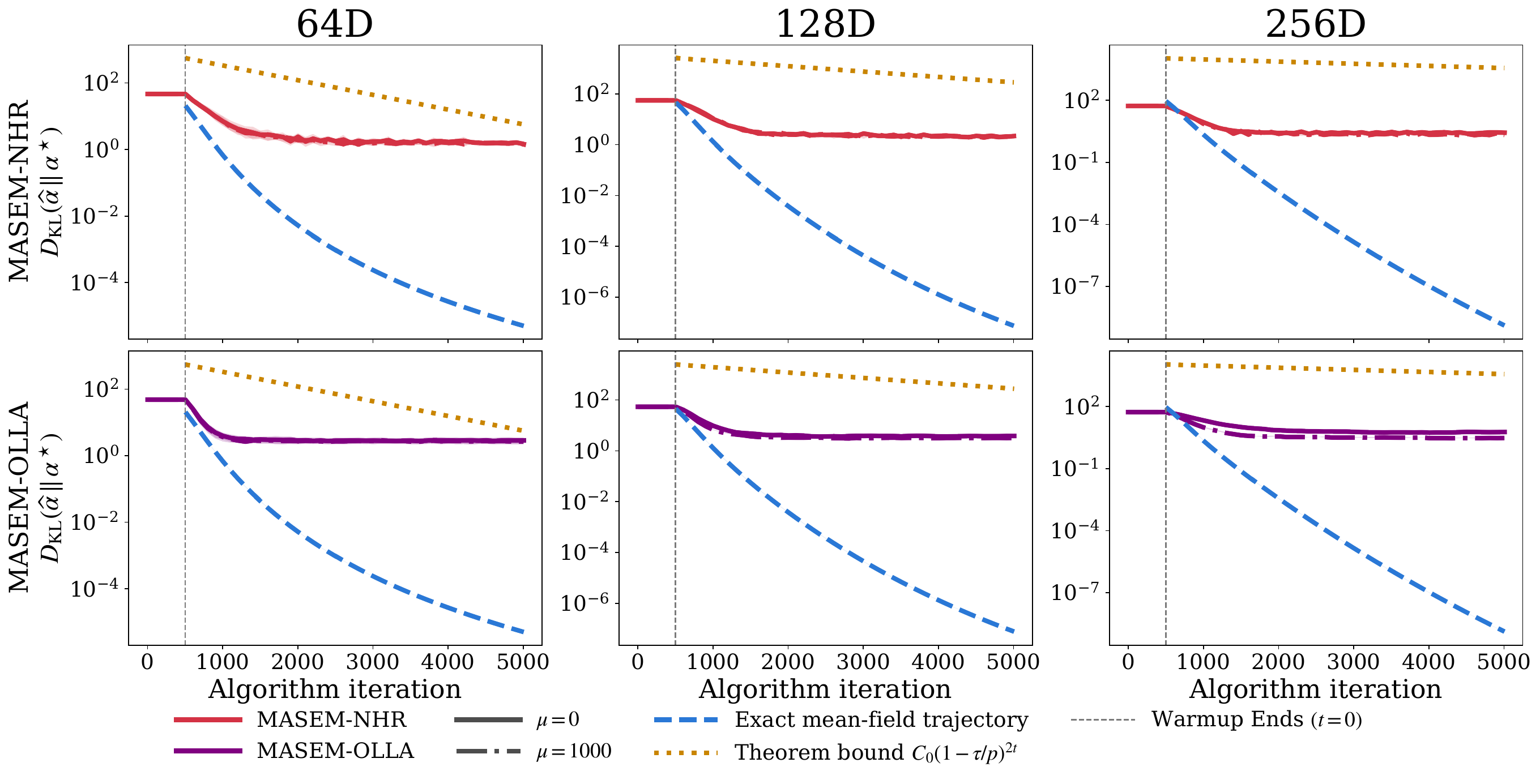}
    \caption{Comparison of empirically observed convergence and exact meanfield rate on the five-component disk stress test problem. Lines show mean and $95\%$ confidence interval across 5 seeds.}
    \label{fig:mf_rate}
\end{figure}

\subsection{Component Loss}\label{app:mode-loss}\
As noted in \Cref{sec:theory-convergence}, there exists a chance that during resampling, none of the chains in a particular component are sampled, meaning this component is lost. Here, we analyze this phenomenon empirically using a synthetic $2$-dimensional benchmark, containing $100$ disks in a $10\times 10$ grid. Two neighboring disk-centers are $5$ units apart; the disks have radius $1$. We run MASEM-NHR for $T=10$ iterations and count the number of components with at least one chain, averaged over ten runs. Importantly, the maximum step size of NHR is set to $1$ to ensure that the chains only jump between components due to resampling. To initialize, we either distribute the chains uniformly across all components (uniform initialization), or populate every component except for one chosen at random starts with a single chain (worst case initialization). 

\begin{figure}[ht]
    \centering
        \includegraphics[width=\linewidth]{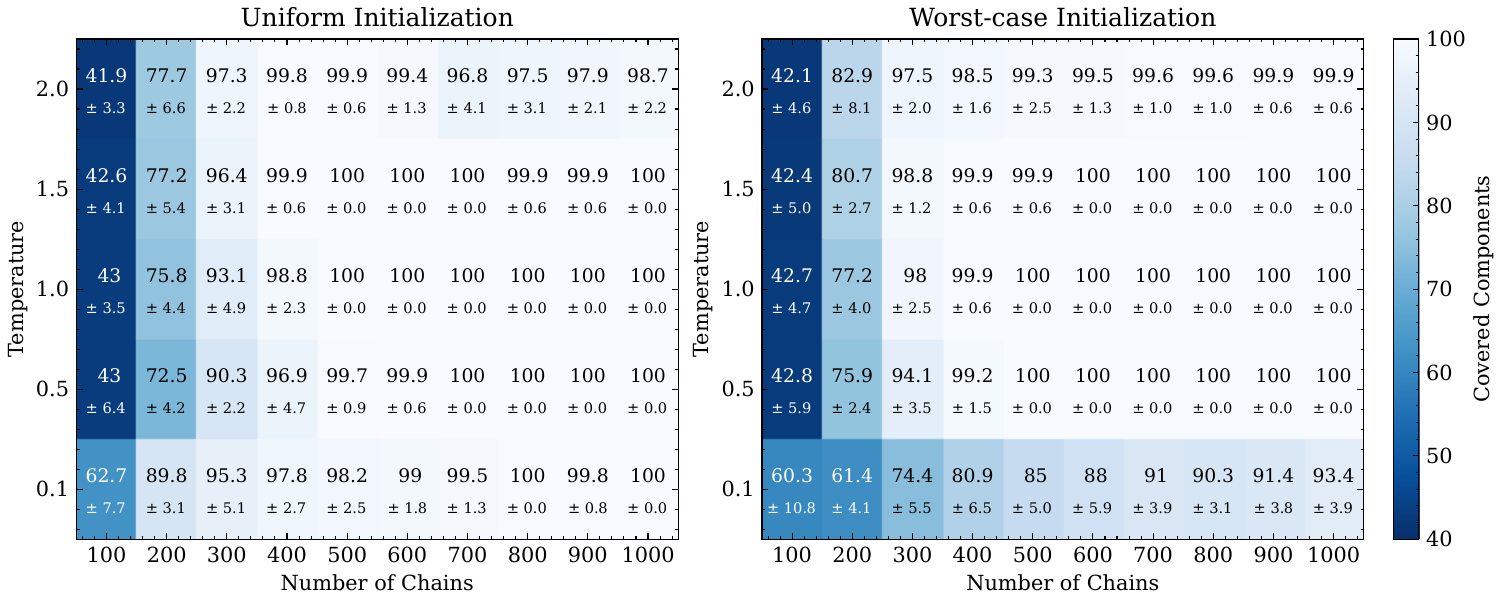}
    \caption{Number of covered components (out of 100), depending on temperature $\tau$ and number of chains $N$. Numbers include $95\%$ confidence interval.}
    \label{fig:mode_loss}
\end{figure}

\Cref{fig:mode_loss} shows the results of this experiment. 
 For a reasonable choice of $\tau\in [0.5, 1.5]$, even just $4$ times the amount of components reduces the chance of component loss over ten iterations to $0.1\%$ at both initializations. Uniform initialization and worst case initialization perform similarly, with the only notable difference being at $\tau=0.1$. There, uniform initialization achieves $99\%$ coverage with $600$ chains, while worst-case initialization does not reach this coverage even at $1000$ chains. 

\subsection{The Pairwise Distance KL-Divergence as Metric}\label{app:pwd-hist-kl}
\begin{wraptable}{R}{0.4\textwidth}
\sisetup{
  round-mode = figures,
  round-precision = 2,
  text-family-to-math = true,
text-series-to-math = true,
scientific-notation = true,
exponent-mode = threshold,
exponent-thresholds = -4:4,
print-zero-integer=false,
detect-all=true
}
    \centering
    \caption{The table shows the probability that  the null hypothesis \enquote{the metric cannot distinguish between a ground truth and a non-ground truth distribution} holds. It is computed with a Welch t-test over $1000$ runs. Note that only $\mathrm{KL}$ rejects the null-hypothesis with $p < 0.01$.}
    \begin{tabular}{lS[table-format=1.2e2]S[table-format=1.2e1]}
    \toprule
$\mathcal{D}$ & {$\mathbb{P}(H_{\mathrm{center}}^\mathcal{D})$} & {$\mathbb{P}(H_{\mathrm{edge}}^\mathcal{D})$}\\
\midrule
$\mathrm{ED}$                          & 0.0678501   & 0.684932   \\
$\mathcal{W}_2^2$                      & 0.0134417   & 0.364927   \\
$\operatorname{KL}$ & \bfseries 4.44049e-34 & \bfseries2.62619e-07\\
        \bottomrule
    \end{tabular}
    \label{tab:KL-welsh}
\end{wraptable}
In our reports, we use the (to our knowledge) novel metric of the KL-divergence computed on the histogram of pairwise distances. We define this metric in \Cref{app:metrics}. Here, we demonstrate the ability of this metric to distinguish between distributions with slightly different densities. 
To that end we sample from a disk ($\rho_\mathrm{GT}$), a disk with a higher weight in the center ($\rho_\mathrm{center}$) and a disk with higher weight at the border ($\rho_\mathrm{edge}$). For a divergence $\mathcal{D}$ we pose the null-hypothesis that the divergence does not distinguish between two instances of the ground truth and a non ground truth density and ground truth. More precisely, our hypothesis $H_{\mathrm{center\mid edge}}^\mathcal{D}$ is that $\mathcal{D}(\rho_\mathrm{GT},\rho_\mathrm{GT})\sim\mathcal{D}(\rho_\mathrm{GT},\rho_\mathrm{center\mid edge})$ holds. 

To measure this, we generate two sets $S_\mathrm{GT1}, S_\mathrm{GT2}$ as well as $S_\mathrm{center}$ and $S_\mathrm{edge}$.  Each set contains $2000$ samples. For $S_\mathrm{center}$, we sample $50$ points of those within a radius of $0.5$ of the center instead of the ground truth. For $S_\mathrm{edge}$, we instead sample $50$ points with radius $r\in [0.5,1.0]$. With these samples we calculate $\mathcal{D}(S_\mathrm{GT1},S_\mathrm{GT2})$, $\mathcal{D}(S_\mathrm{GT1},S_\mathrm{center})$, and $\mathcal{D}(S_\mathrm{GT1},S_\mathrm{edge})$. We repeat this process $1000$ times and use the results to perform a Welch t-test to compute the probability $p$ of rejecting $H_{\mathrm{center}}^\mathcal{D}$ and $H_{\mathrm{edge}}^\mathcal{D}$ respectively.

We compare the energy distance\begin{equation*}
    \mathrm{ED}(X, Y) \coloneqq\left(\frac{2}{nm} \sum_{i, j} \lVert x_i - y_j \rVert - \frac{1}{n^2}\sum_{i, j} \lVert x_i - x_j \rVert - \frac{1}{m^2}\sum_{i, j} \lVert y_i - y_j \rVert\right)^\frac12,
\end{equation*}
used, for example, by \citet{jeon2025fast}, the Sinkhorn distance $\mathcal{W}_2^2$ and our metric ($\operatorname{KL}$). We report the results in \Cref{tab:KL-welsh}. The KL divergence is the only metric which rejects the hypothesis that it cannot distinguish between ground truth and the other densities with $p<0.01$.

\section{Experimental Details}\label{app:exp-details}
Each experiment is repeated across 5 different seeds. 
Where applicable, we report the mean across all 5 runs and 95\% confidence intervals in our plots.
Where reported we test statistical significance using t-tests and Holm-Bonferroni using SciPy.
In the following, we provide details about the benchmarks and implementation that we used in this work.
All the experiments are performed on a single workstation with a NVIDIA RTX PRO 6000 Blackwell GPU, a AMD Ryzen Threadripper PRO 7955WX CPU with 16 cores, and $128$ gigabyte RAM.
All algorithms, benchmarks and metrics are implemented in JAX~\citep{jax2018github}.\footnote{The code is published in \url{https://github.com/LIS-TU-Berlin/masem}.}
The 2d problems are depicted in \Cref{fig:2dbenchs}. 
We will provide mathematical and implementation details on the metrics, the benchmarks and their ground truth samplers (if available) below. After that, we provide details and hyperparameters for the method.

\subsection{Metrics}\label{app:metrics}
The main challenge in comparing different sampling methods is measuring the sampling quality, that is, how well the samples represent the actual distribution. We try to tackle this problem with multiple metrics. For all benchmarks where a ground truth is available, we report the Sinkhorn distance $\mathcal{W}_2^2$. Additionally, we report the average maximum slack violation $\mathbb{E}_x[\max \{|h(x)|, g^+(x)\}]$ for all problems. 

For benchmarks with ground truth, we additionally report the KL divergence of pairwise distance histograms. If no ground truth is available, we instead report the feasible entropy. For the scaling benchmarks we calculate the TV distance between the empirical and the uniform component masses. In the path planning problem, we additionally report the entropy of the homotopy-class membership distribution. All these metrics are defined in more detail below. 

\paragraph{KL divergence.} If a ground truth distribution is available, we report the KL divergence on the histogram of pairwise distances. This is well suited to detect small changes in component mass allocation (\Cref{app:pwd-hist-kl}). For this metric, we compute the pairwise distances of the ground truth distribution as well as the approximate distribution. With this, we compute the normalized histogram with $50$ bins for values clipped between $[0, b]$, where $b =\max_{x,y\in\Sigma}\lVert x-y\rVert_\infty$, i.\,e.~the maximum distance of feasible points along any dimension. We then calculate the forward KL divergence between the ground truth and the approximate distribution.

The pairwise distance KL only compares pairwise distances, and is therefore invariant to relabeling or rotation. 
To be precise, for any $Q\in\mathrm{SO}(d)$ and $t\in\mathbb{R}^d$ we have $\lVert (Q x_i +t) - (Qx_j+t)\rVert = \lVert x_i-x_j\rVert$. 
The distance and therefore the metric itself are in fact invariant to translations, rotations, and reflections. As the metric is also invariant under relabeling, we do not require any information about the components themselves. For samplers that miss some components, the distance distribution changes, which will therefore generally increase the KL with respect to the ground truth. 

\paragraph{Feasible Entropy} If no ground truth is available, we report the feasible entropy, which is the estimated entropy of the feasible samples multiplied with the fraction of feasible samples. This metric serves as a proxy to measure whether we achieve our goal of entropy maximization. We compute 
\begin{equation*}
    \hat H_\mathrm{feas} = \frac{\sum 1[\operatorname{Slack}(x)<\operatorname{tol]}}{N} \cdot
    \frac1{30}
    \sum_{k=2,4,8} 
    \sum_{C_i\subseteq\mathfrak{C}_\mathrm{feas}}   \hat H(C_i, k)\,,
\end{equation*}
where $\mathfrak{C}_{\operatorname{feas}}$ is the set of feasible samples, $C_i$ contains $100$ samples drawn without replacement from $\mathfrak{C}_{\operatorname{feas}}$ for $i=1,\dots 10$ and $\hat H(C_i, k)$ denotes the Kozachenko-Leonenko Entropy estimate of $C_i$, estimated using the $k$-th nearest neighbor \citep{kozachenko1987sample}. 
If less than $100$ feasible samples are available, we do not compute this metric. We choose to subsample $100$ points to compare entropy across methods with potentially different numbers of feasible samples, since the estimator depends on the number of samples.

\paragraph{TV-Distance} 
We introduce a new metric for the scaling problems, since the other metrics might struggle with high-dimensional samples. For this metric we count the number of samples contained in each component (up to a small tolerance of $10^{-5}$) and  normalize it. This yields the component weights $\alpha_c$ as introduced in \Cref{sec:analysis}. We then compute the total variation distance to the ground truth $\alpha^*_c$ as \begin{equation*}
    \mathrm{TV}=\max_{c=1,\dots,C}\lvert \alpha_c -\alpha_c^*\rvert
\end{equation*} 

\paragraph{Homotopy-Entropy.}\label{par:homo-entro}For the motion-planning problems, we additionally measure the entropy of the distribution of paths among their homotopy-classes. Two paths belong to different homotopy classes if they cross an obstacle on different sides. This directly corresponds to the components of the manifold. More formally, we consider the space $W=([-4,4]^2\setminus O) / \sim$, where $O=\bigcup_i O_i$ is the union of obstacles and $(x_i,y_i)\sim(x_j,y_j)$ iff $x_i=x_j=3.6$, which is the vertical goal line. We then consider homotopy classes of paths on this space. Note that, since we constrain paths to have increasing $x$-coordinates, we only observe a finite subset of the homotopy classes. In practice, we compute a signature for each path by testing whether it passes above or below for each obstacle. This results in a $\lvert \mathrm{obs}\rvert$ dimensional vector. We count the occurrences of each signature and compute the entropy of their distribution. Note that while this is a useful metric, maximizing it does not directly correspond to maximizing the entropy of paths. By computing the entropy over the signatures, we implicitly assume that each component has the same weight. While not true in practice, this is still a useful proxy, since the actual component weights are unknown a priori.

\subsection{Synthetic Problems}\label{app:synth-constraints}
For each low-dimensional example, we run $2\,000$ independent chains for $5\,000$ function evaluations. 
From each chain, we retain only the last state, yielding $2\,000$ samples.

\paragraph{Connected \& Disconnected Disks}
We consider two randomly rotated disks embedded on a sphere in $\RRR^3$.
\begin{align*}
h(x) &= \lVert x \rVert_2 - R, \qquad R=2.5;\\
g(x) &= \min_{i\in\{1,2\}}\left\{
\cos(\rho_i) - R\langle x, \mu_i\rangle
\right\}\,.
\end{align*}
We use two different parameters $\rho_1 = 0.2$, $\rho_2 = 0.6$ such that both disks have different areas. 
In the connected case the disks have a arc-length of $\delta=0.6$ between their centers; in the disconnected case we use $\delta=1.35$. We choose the disk centers \begin{equation*}
\begin{aligned}
    \mu_1&=R_x\left(\cos-\tfrac\delta2, \sin-\tfrac\delta2,0\right)^T\\
    \mu_2&=R_x\left(\cos\tfrac\delta2, \sin\tfrac\delta2,0\right)^T
\end{aligned}
    \quad \text{where } R_x=\begin{bmatrix}
       1  & 0  & 0\\
        0 & t &-t \\
        0 & t &t
    \end{bmatrix}\in\operatorname{SO}(3),\quad t=\frac{\sqrt2}2
\end{equation*}is the rotation of $45$ degrees around the x-axis.
In addition we impose bound constraints $x \in [-5,5]^3$.

Ground-truth samples are generated by exact surface-uniform sampling from the union of the two disks, with a correction for possible overlap.
We first choose disk $i\in\{1,2\}$ with probability proportional to its spherical area,
$\mathbb P(i) = \tfrac{1-\cos(\rho_i)}{\sum_{j=1}^2 (1-\cos(\rho_j))}$.
Given disk $i$, we sample uniformly on its area by drawing samples from
\begin{equation*}
    \phi\sim \operatorname{Unif}(0,2\pi), \qquad \cos\alpha\sim \operatorname{Unif}(\cos\rho_i,1),
\end{equation*}
and setting
$
x = R\left[ \cos\alpha\,\mu_i + \sin\alpha \left( \cos\phi\,e_{i,1} + \sin\phi\,e_{i,2} \right) \right]$.
To correct for overlap, proposals are accepted with probability $1/m(x)$, where $m(x)$ is the number of disks containing $x$.

\paragraph{Seven Lobes}
We adapt this problem from \citet{jeon2025fast} and replace the Gaussian mixture target with constant target $f(x) = 0$ to sample uniformly.
\begin{align*}
h(x) &= \sqrt{x_1^2+x_2^2} - \left(3+\cos(7\theta)\right), \qquad \theta=\operatorname{atan2}(x_2,x_1);\\
g(x) &= (x_1-2)^2 - 5x_1x_2^3 + \frac{1}{2}x_2^5 - 40.
\end{align*}
Additionally, we impose bound constraints such that $x \in [-4.5, 4.5]^2$.
We use rejection sampling from the polar parametrization of the equality constraint to generate ground truth samples. 
Specifically, we sample
\begin{equation*}
\theta \sim \operatorname{Unif}(0,2\pi),
\qquad
r(\theta)=3+\cos(7\theta),
\qquad
x(\theta)=r(\theta)(\cos\theta,\sin\theta).
\end{equation*}
We correct the acceptance rate with a Jacobian correction that accepts samples with probability proportional to
\begin{equation*}
    \frac{ds}{d\theta}
    =
    \sqrt{r(\theta)^2+\left(\frac{dr}{d\theta}\right)^2},
    \qquad
    \frac{dr}{d\theta}=-7\sin(7\theta).
\end{equation*}

\paragraph{Sine}
We define a uniform target supported on a one-dimensional sine curve in $\RRR^2$. 
The constraints are
\begin{align*}
    h(x) &= x_2 - \exp(-0.15x_1)\sin(x_1);\\
    x_2 &\geq 0.
\end{align*}
We additionally impose bound constraints such that $x\in[-20,20]^2$.
We generate ground truth samples using rejection sampling from the parametrization of the equality constraint similar to the seven lobes procedure above. 
Specifically, we sample
\begin{equation*}
t \sim \operatorname{Unif}(-20,20), \qquad x(t)=\left(t,\exp(-0.15t)\sin t\right).
\end{equation*}
Since uniform sampling in $t$ is not uniform \wrt arclength, we correct the acceptance probability with the Jacobian.
We retain only accepted proposals satisfying the halfspace and bound constraints. 

\paragraph{Swiss Roll}
We define a uniform target on a disconnected one-dimensional manifold in $\RRR^2$ consisting of six circles and one Archimedean spiral segment.
The circle components are
\begin{equation*}
    h_k(x)=\lVert x-c_k\rVert_2-r_k,\qquad k=1,\ldots,6,
\end{equation*}
where the centers $c_k$ and radii $r_k$ are sampled once with layout seed $0$,
using non-overlapping circles with radii in $[0.45,1.15]$.
The spiral is parametrized by
\begin{equation*}
    \gamma(t) = c_{\mathrm{sp}} + (a+bt)(\cos t,\sin t), \qquad t\in[t_{\min},t_{\max}],
\end{equation*}
with
\begin{equation*}
    c_{\mathrm{sp}}=(4.8,-0.4),\qquad
    a=0.45,\qquad
    b=0.33,\qquad
    t_{\min}=0.9,\qquad
    t_{\max}=3.9\pi.
\end{equation*}
The equality constraint is the signed residual of the closest circle or valid
spiral branch. 
We additionally impose
\begin{equation*}
    g(x)= \operatorname{softplus}(-3-x_1) \operatorname{softplus}(-x_2) - 0.1,
\end{equation*}
as well as box constraints $x\in[-10,10]^2$.
Ground-truth samples are generated by first choosing a circle or the spiral with probability proportional to its parameter-domain length, i.e. $2\pi$ for each circle and $t_{\max}-t_{\min}$ for the spiral. 
For circles we sample $\phi\sim\operatorname{Unif}(0,2\pi)$ and set $x=c_k+r_k(\cos\phi,\sin\phi)$, correcting by the arclength Jacobian. 
For the spiral we sample $t\sim\operatorname{Unif}(t_{\min},t_{\max})$ and set $x=\gamma(t)$, correcting by $\frac{ds}{dt}=\sqrt{(a+bt)^2+b^2}$.

\paragraph{Scaling Stress Test}
In the synthetic stress-test benchmark we target an uniform disconnected manifold embedded in $\RRR^d$.
We use $m$ to denote the total number of equality constraints.
Of these, $m-1$ are linear equality constraints, which define a random subspace via
\begin{equation*}
    A x = b,
\end{equation*}
where $A\in\RRR^{(m-1)\times d}$ has orthonormal rows and is sampled randomly once.
In all stress-test experiments, we set $b=0$.
Let
\begin{equation*}
    p = d - (m-1) = d-m+1
\end{equation*}
denote the dimension of the nullspace of $A$, and let $N\in\RRR^{d\times p}$ be an orthonormal basis of this nullspace.
We write the corresponding latent coordinates as
\begin{equation*}
    z = N^\top x .
\end{equation*}

Within this subspace we then randomly sample $|C|=5$ disconnected hyperspheres which define the manifold on which we sample via equality constraints.
For centers $c_i\in\RRR^p$ and radii $r_i$, we use
\begin{equation*}
    h(x)
    =
    \left\lVert z-c_{i^\star}\right\rVert_2^2-r_{i^\star}^2,
    \qquad
    i^\star
    =
    \argmin_{i\in\{1,\dots,5\}}
    \left|
        \left\lVert z-c_i\right\rVert_2^2-r_i^2
    \right|.
\end{equation*}
Since each sphere has dimension $p-1$, the intrinsic manifold dimension is
\begin{equation*}
    p-1 = d-m .
\end{equation*}
The base radius is set to $0.25$, and the component radii are multiplied by independent jitter factors in $[0.5,1.5]$.
The component centers are placed randomly in latent space using rejection sampling to ensure separation between components.

We additionally include exactly $l=5$ spherical inequality constraints, which act as ball cutouts on the sphere components.
Each cutout center is placed on the surface of a randomly selected component, and its radius is set to half of the corresponding component radius.
For cutout center $o_j$ and radius $\rho_j$, the inequality is
\begin{equation*}
    g_j(x)
    =
    \rho_j^2 - \left\lVert z-o_j\right\rVert_2^2,
    \qquad j=1,\dots,5,
\end{equation*}
so that feasible samples lie outside all cutout balls.
Finally, we impose box constraints $x\in[-28,28]^d$.

We compare two main scaling regimes: increasing ambient dimension with fixed manifold dimension for $m=d-3$ and increasing ambient dimension with increasing manifold dimension for fixed $m=5$.
In this case, the intrinsic dimension grows with the ambient dimension and is given by $d-m=d-5$.

Ground-truth samples are generated directly from the latent representation.
We first sample a component index with probability proportional to its surface area,
\begin{equation*}
    \mathbb P(i)
    =
    \frac{r_i^{p-1}}{\sum_{j=1}^5 r_j^{p-1}}.
\end{equation*}
Given the component, we sample a direction $u$ uniformly on the unit sphere $\mathbb S^{p-1}$ by normalizing a standard Gaussian vector, and set
\begin{equation*}
    z = c_i + r_i u,
    \qquad
    x = Nz .
\end{equation*}
Since all stress-test experiments include ball cutouts, we use rejection sampling from this proposal and retain only feasible samples.

\subsection{Motion Planning}\label{sec:mp_details}
We consider 2D motion planning problems in which a point robot must navigate from a fixed start configuration $q_0 \in \RRR^2$ to a goal region while avoiding circular obstacles.
We parametrize each path through $n_w$ waypoints $X = (x_1,\dots,x_{n_w}) \in \RRR^{n_w \times 2}$, which form the sampling domain.
Following standard robotics convention, we recover smooth trajectories as cubic B-splines whose control points are obtained from the start and waypoints via $C = \Phi^{-1}[q_0; X]$, using the pseudoinverse of the interpolation matrix $\Phi \in \RRR^{T \times (n_w+1)}$.
Evaluating the resulting curve at $T$ equally spaced parameter values yields a discretized path $p \in \RRR^{T\times 2}$.
The goal is a vertical line on the right side of the workspace, which we impose as an equality constraint on the first coordinate of the final waypoint:
\begin{equation*}
    h^{\text{goal}}(X) = x_{n_w,1} - 3.6,
\end{equation*}
leaving the second coordinate free so any point on the goal line is admissible.

The main challenge of this domain is navigating around obstacles, where multiple homotopically distinct paths are admissible.
Collision constraints are enforced at each path sample via
\begin{equation*}
    g^{\text{coll}}_{j,t}(p) = r_j^2 - \lVert p_t - c_j \rVert_2^2, \qquad t=1,\dots,T,
\end{equation*}
for each obstacle $j$ with center $c_j$ and radius $r_j$.
We additionally impose monotone progress along the $x$ axis,
\begin{equation*}
    g^{\text{mon}}_t(p) = p_{t-1,1} - p_{t,1}, \qquad t=1,\dots,T,
\end{equation*}
with $p_0 \equiv q_0$, which removes looping and backwards-going trajectories.
We further impose discrete velocity and acceleration limits to suppress paths with large gaps or sharp kinks, both of which are known failure modes when the path is reconstructed from sparse waypoints:
\begin{align*}
    g^{\text{vel}}_t(p) &= \lVert p_t - p_{t-1} \rVert_2^2 - s^2_{\max}, & s_{\max} &= 1, \\
    g^{\text{acc}}_t(p) &= \lVert p_{t+1} - 2p_t + p_{t-1} \rVert_2^2 - a^2_{\max}, & a_{\max} &= 0.65.
\end{align*}
Finally, all waypoint coordinates are bound-constrained to $[-4, 4]^2$.
Since each obstacle splits the configuration space into a homotopy class passing above and one passing below, the feasible set decomposes into up to $2^{|\text{obs}|}$ connected components (modulo monotonicity and feasibility pruning), which we use to evaluate homotopy-class entropy. For more details see \Cref{par:homo-entro}.

\paragraph{Grid Layout.}
In the grid problem, obstacles are placed on a regular $4 \times 4$ grid, that is centered within the domain. 
All obstacles share a common radius $r=0.5$. 
In total, this induces $m=3$ equality and $l=771$ inequality constraints.
This setup yields a structured benchmark with a known and combinatorially large number of homotopy classes, which is well suited for testing whether a sampler covers all feasible passages between obstacle columns or collapses onto a few preferred routes.

\paragraph{Random Layout.}
In the random problem, we place 20 disk obstacles with radii drawn independently from $\operatorname{Unif}[0.2, 0.5]$ and center coordinates drawn uniformly over the domain.
We employ rejection sampling to guarantee minimal separation between obstacle pairs to increase the number of feasible paths.
In total, this induces $m=3$ equality and $l=931$ inequality constraints.
The random layout induces an irregular and often narrower set of feasible corridors than the grid, reflecting real world robot navigation challenges.

\subsection{Grasping}
We want to grasp a capsule, which is a cylinder with a halfsphere with the same radius as the cylinder on each end. This shape is commonly used for collision checking. The capsule is placed at the origin $(0,0,0)$. Its orientation is given by a vector along its long axis, pointing towards $o=(1,1,1)$. It has the length $l=1$ and the radius $r=0.25$. Gravity impacts this capsule with the force-vector $f_g = (0,0,1)$. Its friction coefficient is $\mu=1$. \Cref{fig:grasp_diagram} shows the capsule and a grasp schematically.

\begin{wrapfigure}{R}{0.3\textwidth}
    \centering
    \includegraphics[width=0.8\linewidth]{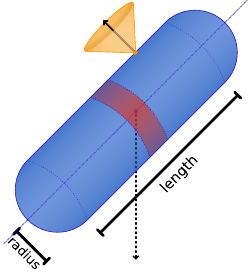}
    \caption{Schematic diagram visualizing the capsule in $2$d. The red band in the middle represents the infeasible region. In the top left, a finger contact is shown. The orange dot is the point of attack on the capsule surface. The orange cone shows the friction cone and the black arrow the surface normal. The dotted black arrow in the middle symbolizes gravity.}
    \label{fig:grasp_diagram}
\end{wrapfigure}

We grasp the capsule with three fingers. The fingers as hard point-fingers with friction \cite{murray1994mathematical}. This means each finger $i$ is parametrized by a point of attack $p_i\in\mathbb{R}^3$ as well as a force-vector $f_i\in\mathbb{R}^3$. To ensure the capsule is grasped and stable, the sum of forces of the fingers has to counter gravity, and the sum of torques at it's center $\tau_i = p_i\times f_i$ has to be zero. This yields the constraints
\begin{align*}
    f_1 + f_2 + f_3 &= f_g\\
    \tau_1 + \tau_2 +\tau_3 &= 0\ .
\end{align*}

For each finger $i$, we calculate a vector pointing from the middle axis of the capsule towards the finger as \begin{equation*}
    d_i = \operatorname{clamp}\bigl(\langle o,p_i\rangle, -l/2, l/2\bigr)o - p_i\,.
\end{equation*}
This allows us to ensure that the point of attack lies on the capsules surface by constraining $\lVert d_i\rVert= r$.
Furthermore, each finger can exert a force deviating from the normal of the capsules at the point of attack $n_i=d_i / r$ as far as friction allows. This is expressed in the force cone constraint
\begin{equation*}
\lVert f_i - n_i\langle n_i, f_i\rangle\rVert\leq \mu \langle -n_i, f_i\rangle\,.
\end{equation*} 
We constrain the force each finger can apply to $\lVert f_i\rVert\in[0.1,1]$. Lastly, we introduce an infeasible region on the capsule. This is meant to model objects that can't be grasped at arbitrary points. We choose a band across the middle of the cylinder with width $w=0.25$. This results in the constraint $\rvert\langle o,p_i\rangle\lvert\geq w/2$.

For each finger, we have $6$ dimensions, $1$ equality, and $4$ inequality constraints. We enforce $6$ equalities in the force and torgue constraints. Additionally, we constrain each entry of the points of attacks and forces to be between $-1$ and $1$. Overall, this results in $18$ dimensions, $9$ equalities and $48$ inequalities.

\subsection{Implementation Details}\label{app:impl-details}
All methods are implemented in JAX~\citep{jax2018github}.
We implement all baselines based on their official implementation: NHR is based on the code of \citet{toussaint2024nlp}, OLLA based on \citet{jeon2025fast} and SCMC on \citet{golchi2016sequentially}.
For OLLA, we use the OLLA-H variant which employs Hutchinson's trace estimator.

\paragraph{Initialization.} 
We initialize all sampling runs from a Gaussian in ambient space which we center at the origin and scale by a quarter of the distance between minimum and maximum bound for that problem.
Following \citet{toussaint2024nlp}, we implement the initial projection for NHR and MASEM-OLLA using the Gauss-Newton method for 500 steps with additional standard Gaussian noise that we scale by $\epsilon=0.01$.
While this is inconsistent with theory, we find that it improves mode coverage of the initialization and yields feasible samples on every problem that we tested on.

\paragraph{MASEM Details.}
In \Cref{sec:analysis}, we assume that the local sampling kernel preserves sample feasibility. 
Since practical samplers tend to break this property resampling naively can increase the number of infeasible samples as those tend to have largest distance to the samples on the manifold.
To prevent this from happening, we apply a penalty from the resampling weights:
\begin{equation*}
    w_i \gets w_i\frac1{\exp\!\big(\mu\operatorname{Slack}(x_i)\big)}
\end{equation*}
We choose $\mu=1000.$ as default and find that this offers a robust choice.
Further, we ensemble multiple values for $k$ to compute the weights~\citep{laskin2022cic, 2019berret}.
Specifically for $k=3$ we would compute all distances $\varepsilon_{i,1}, \varepsilon_{i,2}, \varepsilon_{i,3}$ and take the geometric mean over their resulting logarithmic distance $\bar{w}_i = \bigl(\tfrac{1}{k}\tsum_{j=1}^k \log \varepsilon_{i,j}\bigr)^\tau$.
Specifically, for fixed $k$ we take the geometric mean  the $k$ nearest neighbors distances $\bar{w}_i = \bigl(\prod_{j=1}^k\varepsilon_{i,j}\bigr)^{\frac{\tau}{k}}$.
This approach tends to have lower variance and is less sensitive to the exact choice of $k$ despite introducing bias.

\paragraph{Hyperparameters}
We tune the main hyperparameters of each method and baseline using Bayesian Optimization \cite{snoek2012practical}, implemented in the wandb api.\footnote{\url{https://docs.coreweave.com/products/wandb/sweeps}}
For each method we allocate the same budget of 20 tuning runs which we tune against the Sinkhorn distance on problems with ground truth samples and against the feasible entropy on practical problems.
The MASEM hyperparameters for all problems are listed in \Cref{tab:masem_hparams}.

\begin{table}[ht]
    \centering
    \sisetup{
print-zero-integer=false
}
\caption{Hyperparameters for MASEM-NHR and MASEM-OLLA. Step sizes and other sampler-related hyperparameters are taken from the corresponding base sampler. Hyperparameters are the temperature $\tau$, the number of mixing steps $M$ and the considered neighbour $k$.}
    \begin{tabular}{lS[table-format=1.2]S[table-format=1.0]S[table-format=1.0] c
    S[table-format=1.2]S[table-format=1.0]S[table-format=1.0]}
      \toprule
      & \multicolumn{3}{c}{MASEM-NHR} & & \multicolumn{3}{c}{MASEM-OLLA} \\
      \cmidrule(lr){2-4} \cmidrule(lr){6-8}
      Problem & {$\tau$} & {$M$} & {$k$} & & {$\tau$} & {$M$} & {$k$} \\
      \midrule
      Connected Disks & 1.0 & 50 & 4 & & 1.0 & 50 & 4 \\
      Disconnected Disks & 1.0 & 50 & 4&  & 1.0 & 50 & 4 \\
      Seven Lobes & 0.81 & 5 & 8 & & 0.57 & 5 & 8 \\
      Sine & 0.98 & 5 & 16 & & 0.94 & 10 & 16 \\
      Swiss Roll & 0.65 & 5 & 16& & 0.5 & 50 & 8 \\
      Stress Test ($m=d-3$) & 0.3 & 5 & 16 & & 0.3 & 5 & 16 \\
      Stress Test ($m=5$) & 0.74 & 10 & 20 & & 0.74 & 10 & 20 \\
      \midrule
      Motion Planning Grid & 0.65 & 50 & 8 & & 0.65 & 100 & 8 \\
      Motion Planning Random & 1.0 & 50 & 8 & &0.94 & 50 & 4 \\
      Grasping & 0.41 & 5 & 19 & & 0.94 & 10 & 16\\
      \bottomrule
    \end{tabular}
    \label{tab:masem_hparams}
\end{table}


\section{Limitations}\label{app:limitations}
\subsection{Theoretical Limitations}
Some results, namely \Cref{lem:inital-populated}, \Cref{thm:main}, and \Cref{cor:convergence}, are valid only in the large-scale particle limit, i.e.~for $N\to\infty$. Furthermore, the main result in \Cref{thm:main} is only stated in the mean-field. 
It might happen that no particles are sampled for a particular component, resulting in that component vanishing.
However, the probability of this occurrence decreases exponentially with the number of samples and our experiments show this effect is negligible even for just $4$ chains per component, as we demonstrated in \Cref{app:mode-loss}.

\subsection{Runtime Overhead}
While small, our method introduces some runtime overhead compared to only using constrained samplers like NHR: every $M$ steps, we compute the $k$-nearest neighbor distances and resample. This is also reflected in the asymptotic runtime. Assuming each iteration of the constrained sampler is constant in the number of samples, we increase the asymptotic runtime from $\mathcal{O}(NTM)$ to $\mathcal{O}(NT(M+\log N)$ when using the proposed resampling.\footnote{Using approximate nearest neighbors, one can determine the k-nearest neighbor distance in $\mathcal{O}(N\log N)$ \citep{harwood2025approximate}. Furthermore, with binary search, the resampling step can be executed in $\mathcal{O}(N\log N)$ as well}
In practice however, we discovered that on our benchmarks this overhead is small, while applying MASEM greatly enhances sample quality.

\subsection{Non-Uniform Sampling}
This work focuses on entropy maximization, which results in uniformly sampling the feasible set. This limits it's usefulness when non-uniform samples are needed. If that is the case, we suggest either applying MASEM and using rejection sampling afterwards, or using the output of MASEM as an unbiased initialization for a Sequential Monte Carlo sampler \citep{andrieu2003introduction}.

\subsection{Parameter Selection}
As shown in \cref{thm:main,cor:convergence}, the convergence of the algorithm depends on the temperature $\tau$, which should be chosen to be smaller than the intrinsic dimension of the manifold. For many applications, the intrinsic dimension is not known a priori, making it harder to select a suitable parameter. We suggest a few mitigation strategies. First, if the available compute resources allow, one could test multiple different parameters and compare the resulting entropy, as we did for tuning the robotics benchmarks. Secondly, especially in robotics, it is feasible to obtain a rough estimate of the intrinsic dimension by reasoning over the degrees of freedom of the problem. This could then be used as an upper bound. Lastly, if both strategies are infeasible, we suggest using a conservative temperature (for example, $\tau=0.9$) and increasing the number of iterations $T$.

\section{Broader Impact}\label{app:impact}
This work is foundational, and we expect no direct societal impact. Nevertheless, as described in \cref{sec:intro}, constrained sampling has many applications in robotics. An important concern when deploying robots in the real world is safety, especially when they are controlled by a neural network. We believe safe operation can be encouraged during training by sampling from a constrained distribution, where the constraints ensure, for example, that the robot does not collide with the environment. The method presented here should be only one of many components of a holistic system for safe robot operation. While we hope it can play a small part in improving robotic operations, we believe that its direct broader impact is limited.

\section{List of Acronyms}
\printacronyms[display=all,heading=none,template=tabularx]


\end{document}